\documentclass{article} 
\usepackage{iclr2027_conference,times}

\usepackage{amsmath,amsfonts,bm}

\def\eqref#1{equation~\ref{#1}}

\def\1{\bm{1}}

\DeclareMathAlphabet{\mathsfit}{\encodingdefault}{\sfdefault}{m}{sl}
\SetMathAlphabet{\mathsfit}{bold}{\encodingdefault}{\sfdefault}{bx}{n}

\newcommand{\R}{\mathbb{R}}

\DeclareMathOperator{\Tr}{Tr}

\newcommand{\be}{\boldsymbol{e}}
\newcommand{\bff}{\boldsymbol{f}}
\newcommand{\bg}{\boldsymbol{g}}

\newcommand{\bw}{\boldsymbol{w}}
\newcommand{\bx}{\boldsymbol{x}}

\newcommand{\bdelta}{\bm{\delta}}

\newcommand{\btheta}{\boldsymbol{\theta}}

\newcommand{\bxi}{\bm{\xi}}

\newcommand{\cO}{\mathcal{O}}

\usepackage[utf8]{inputenc} 
\usepackage[T1]{fontenc}    
\usepackage{titletoc}
\usepackage[pagebackref]{hyperref}   
\usepackage{url}            
\usepackage{booktabs}       
\usepackage{amsfonts}       
\usepackage{nicefrac}       
\usepackage{microtype}      
\usepackage{xcolor}         
\usepackage{wrapfig}
\usepackage[capitalize,noabbrev]{cleveref}
\usepackage{lipsum}

\usepackage[ruled,vlined]{algorithm2e}

\hypersetup{colorlinks=true,linkcolor=red!70!black,linktocpage=false,citebordercolor=blue!70!black,citecolor=blue!70!black,anchorcolor=blue!70!black}

\usepackage{graphicx,subfig}
\usepackage{etoc}
\usepackage{multirow}
\usepackage{makecell}
\usepackage{stfloats}
\usepackage{tcolorbox}
\usepackage{framed}
\colorlet{shadecolor}{orange!15}

\usepackage{amsmath}
\usepackage{amssymb}
\usepackage{mathtools}
\usepackage{amsthm}

\theoremstyle{plain}
\newtheorem{theorem}{Theorem}[section]
\newtheorem{proposition}{Proposition}[section]
\newtheorem{lemma}{Lemma}[section]

\theoremstyle{definition}

\newtheorem*{main result}{Main Theorem}
\allowdisplaybreaks[4]

\usepackage{enumitem}

\usepackage{wrapfig}

\newcommand{\norm}[1]{\left\lVert#1\right\rVert}

\usepackage{enumitem}

\title{How Does Local Landscape Geometry Evolve in Language Model Pre-Training?}

\author{Zhanpeng Zhou\thanks{Equal contribution. Correspondence to Zhanpeng Zhou (\texttt{zzp1012@sjtu.edu.cn}). 
}$^{~~ 1}$, 
Yuhan Sun$^{*1}$, 
Bingrui Li$^{2}$,
Jinbo Wang$^{3}$,
Huaijin Wu$^{1}$,\\
\textbf{Lei Wu$^{3}$},
\textbf{Junchi Yan$^{1}$}
\\
$^1$Shanghai Jiao Tong University
$^2$Tsinghua University 
$^3$ Peking University
}

\iclrfinalcopy 
\begin{document}

\maketitle

\begin{abstract}
The scale and expense of pre-training language models make efficient hyperparameter tuning essential, yet a principled guidance is still missing. 
In this work, we analyze language model pre-training dynamics from a local landscape geometry perspective.
Our study reveals two distinct phases.
In Phase I, sharpness of the local landscape is initially high, leading to instability and loss plateaus under large learning rates (LRs).
The landscape shifts from sharp to flatter regions early in training. 
This dynamic explains the necessity of LR warmup and further suggests that larger peak LRs require proportionally longer warmup periods. 
In Phase II, the local landscape is governed by the gradient noise scale.
Our theory identifies a depth flatness trade-off: high noise from smaller batches widens the loss basin, whereas reduced noise from larger batches deepens it. 
This theory motivates a dynamic batch-size (BS) scheduler that begins with a small BS and increases it late in training. 
Together, we provide a unified view of loss landscape evolution, which translates into actionable tuning strategies for large-scale pre-training.
\end{abstract}

\section{Introduction}

Training language models efficiently requires carefully tuned hyperparameters, yet a principled guidance for tuning remains unclear. While practitioners often rely on grid search, these approaches are costly and unreliable at scale. Recent research~\citep{foret2021sharpnessaware,cohen2021gradient,gilmer2022a} has highlighted that the geometry of the local loss landscape offers fundamental insights into optimization, revealing how factors such as sharpness\footnote{To avoid misunderstanding, we clarify the terminology in \cref{tab:term}.} ~\citep{keskar2017on,zhang2017understanding,Jiang2020Fantastic} shape training dynamics.
Leveraging insights from the local loss landscape presents a promising path toward principled hyperparameter tuning for language model pre-training.

Several pioneering works have already attempted to study language models from the local loss landscape perspective.
\citet{zhang2024why,wang2025the} identified blockwise sharpness patterns in language models through Hessian-based analyses.
\cite{wen2024understanding} introduced the ``river-valley'' landscape to explain the effectiveness of Warmup-Stable-Decay (WSD) schedules~\citep{hu2024minicpm}.
\citet{peng2024navigating,chen2025understanding} further visualized the loss landscapes of finetuned language models, offering geometric insights into the safety alignment.

To this end, we pose the central questions of this paper:
\begin{itemize}[leftmargin=2.5em,topsep=1pt,itemsep=1.5pt,partopsep=1.5pt, parsep=1.5pt]
    \item[\textbf{{\color{black} 1.}}] \emph{How does the local landscape geometry evolve in language model pre-training?}
    \item[\textbf{{\color{black} 2.}}] \emph{What implications does this evolution have for principled hyperparameter tuning?}
\end{itemize}

\textbf{Our contributions.}
In this work, we present the systematic study of the evolution of local landscape geometry during language model pre-training. 
As illustrated in \cref{fig:fig1}, our analysis reveals two distinct phases, each with implications for hyperparameter tuning. 

\textbullet~\emph{Phase I: From Sharp to Flat Landscapes.} 
In Phase I, we observe that the model shifts from sharper regions of the loss landscape toward flatter ones, contrary to the progressive sharpening phenomenon in prior works~\citep{cohen2021gradient,song2023trajectory,cohen2025understanding}.
Lyapunov stability analysis in \cref{sec:early} shows that the maximum stable learning rate (LR) is inversely proportional to sharpness.
Since sharpness is extremely high early in pre-training, using large peak LRs without sufficient warmup leads to instabilities, such as loss spikes and plateaus (see \cref{fig:instability and loss plateau}). 

\textbf{Implications.}
The sharp-to-flat transition explains the necessity of LR warmup: LR should remain small until sharpness has sufficiently decayed, preventing training instabilities. 
This further provides a tuning recipe: within a reasonable range, larger peak LRs require proportionally longer warmup, to safely navigate the sharpest stage of training.

\textbullet~\emph{Phase II: Basin Selection Governed by Noise Scale.} 
In Phase II, the local loss landscape is largely governed by the noise scale during training, with batch size (BS) serving as its primary controller. 
Our analysis shows that smaller BS widens the loss basin, while larger BS deepens it.
Theoretically, we analyze a continuous setup of preconditioned SGD, which uncovers a depth flatness trade-off: reduced gradient noise tends to minimize the loss, leading to deeper minima; whereas increased noise tends to regularize the sharpness of landscape, moving toward wider ones.

\textbf{Implications.}
The trade-off, together with the ramping-time experiment in \cref{fig:timing}, motivates a BS scheduling strategy: begins with a small BS and ramps it until the late phase of training.
Our scheduling ensures steady loss reduction with few token consumption. 
Moreover, since the noise scale is proportional to $\eta / B$ in our theory, we predict that BS ramping and LR decay reduce the noise scale in similar ways and thus yield comparable performance (see \cref{fig:discussion}).

In summary, our work provides a two-phase picture of pre-training: an early sharp-to-flat transition that necessitates LR warmup, and a late noise-driven regime that motivates BS scheduling. 
This unified view advances our understanding of pre-training dynamics.

 \begin{figure}[tb!]
    \centering
    \includegraphics[width=0.80\textwidth]{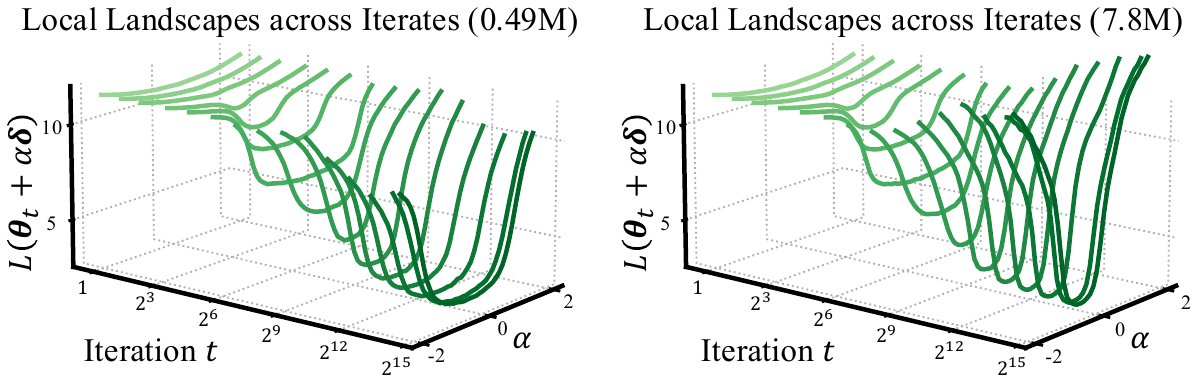}
    \caption{
    \textbf{The evolution of local loss landscape throughout pre-training.}
    We train \texttt{LLaMA-2} models with $170$M parameters using different BSs ($0.49$M and $7.8$M), and visualize the one-dimensional loss landscape at iterate $\btheta_t$ along a random direction $\bdelta$, i.e., plot $L(\btheta_t + \alpha \bdelta)$ vs. the perturbation coefficient $\alpha$.
    The landscapes are shown across different training iterations $t$.
    \textbf{Phase I.} The landscapes gradually widen for both training runs.
    \textbf{Phase II.} Training with smaller BS produces wider landscapes than training with larger BS.
    }
    \label{fig:fig1}
\end{figure}

\section{Related Works}

\textbf{Local landscape geometry evolution.}
Understanding how local landscape geometry, particularly sharpness, evolves during training has drawn attention before the success of large language models.
\citet{wu2018sgd,cohen2022adaptive,song2023trajectory,cohen2025understanding} showed that initially gradient descent (GD) tends to move from flatter to sharper regions of the landscape.
In addition, \citet{jastrzębski2018on,Jastrzebski2020The} argued that in SGD, sharpness also changes monotonically but either increase or decrease depending on the setting.
In the later phase, however, sharpness is largely governed by the properties of the optimizer~\citep{zhou2025sharpnessaware}.
One notable example is that the stochastic noise of SGD and its variants implicitly biases training toward flat minima~\citep{wu2018sgd,pmlr-v97-zhu19e,xie2021a,wu2022the}.
Yet, these findings are largely restricted to small networks;
\emph{In comparison}, our work presents a \emph{systematic study} of how local landscape geometry evolves in large-scale language model pre-training.

\textbf{Large-scale pre-training: learning rate warmup.}
Learning rate warmup, first introduced in large-batch ResNet~\citep{he2016deep,goyal2017accurate} and Transformer training~\citep{vaswani2017attention}, is now standard in large-scale pre-training~\citep{shoeybi2019megatron,zhang2022opt,hu2024minicpm}.
Its mechanism, however, remains poorly understood.
\citet{gotmare2018a} showed that warmup prevents excessively large early parameter updates;
\citet{bergsma2025straight} attributed the early updates to bias reduction.
\citet{kosson2024analyzing} showed in language model pre-training that warmup mitigates momentum bias correction and correlated gradients that drive unstable representation shifts.
Yet no unified explanation exists.
\emph{In comparison}, our work views warmup from a geometric perspective, suggesting that larger peak LRs require proportionally longer warmup.

\textbf{Large-scale pre-training: batch size schedules.}
Batch size is another critical hyperparameter in large-scale pre-training, controlling the trade-off between time efficiency and data efficiency.
Prior works~\citep{mccandlish2018empirical,kaplan2020scaling,gray2023efficient,gray2024normalization,zhang2024does} have focused on the critical batch size (CBS), the point where further increasing BS yields diminishing returns.
However, CBS is usually considered as constant throughout training, and less attention has been given to BS scheduling.
Early works on adaptive sampling proposed gradually increasing BS to balance efficiency and noise reduction~\citep{de2017automated,lau2024adadagrad,lau2024communication,lau2024adaptive,ostroukhov2024adabatchgrad}.
Advanced language models~\citep{brown2020language,touvron2023llama,liu2024deepseek,li2025minimax} employed stage-wise BS schedules, but without systematic analysis.
\emph{In contrast}, our work connects BS scheduling to the evolving local loss landscape, providing a principled way for when to increase the BS.

\section{Preliminaries}\label{sec:prelim}

\textbf{Basic notations.}
We use bold lowercase letters (e.g., $\bx = (x_i)$) to denote vectors and bold uppercase letters (e.g., $\mathbf{A} = (a_{ij})$) to denote matrices.
For a matrix $\mathbf{A}$, let $\norm{\mathbf{A}}_2$, $\norm{\mathbf{A}}_F$, and $\Tr(\mathbf{A})$ denote its spectral norm, Frobenius norm and trace, respectively.
The Hadamard product is denoted by $\odot$.

\textbf{Theoretical setup.}
Our theory focuses on the preconditioned stochastic gradient descent (PSGD).
We consider a model with parameters $\btheta \in \R^p$ and a training set of $n$ examples.
Let $L_i(\btheta)$ be the fitting error evaluated at the $i$-th example and $L(\btheta) = \frac{1}{n}\sum_{i=1}^n L_i(\btheta)$ be the empirical risk.
We analyze the preconditioned SGD with a fixed positive-definite\footnote{Most practical preconditioners are positive-definite: $\mathbf{M} = I$ for SGD, diagonal $\mathbf{M}$ for AdaGrad~\citep{JMLR:v12:duchi11a}, RMSProp~\citep{tieleman2012rmsprop}, Adam, etc.} preconditioner $\mathbf{M} \succ 0$.
At iteration $k$, the update rule gives: \begin{align}\label{eq:noisy dynamics}
    \btheta_{k+1} = \btheta_k - \eta \mathbf{M} (\nabla L(\btheta_k) + \bxi_k),
\end{align}
where $\eta > 0$ is the LR and $\{\bxi_k\}$ are i.i.d. random noise vectors with \begin{align}
    \mathbb{E}[\bxi_k] = \boldsymbol{0}, \quad \mathbb{E}[\bxi_k\bxi_k^\top] = \mathbf{\Sigma}(\btheta_k)/B.
\end{align}
Note that $\mathbf{\Sigma}(\btheta_k)=\frac{1}{n}\sum_{i=1}^n \nabla L_i(\btheta_k) \nabla L_i(\btheta_k)^{\top} - \nabla L(\btheta_k) \nabla L(\btheta_k)^{\top} $ is the gradient covariance at $\btheta_k$, and $B$ denotes the BS.
During the late phase of training, the model remains close to some global minimum $\btheta^\star$ and the loss can be approximated quadratically: \begin{align}\label{eq:quadratic approximation}
    L(\btheta) &= L(\btheta^\star) + \frac{1}{2}{(\btheta - \btheta^\star)}^\top \mathbf{H}(\btheta^\star) {(\btheta - \btheta^\star)}, \quad \mathbf{H}(\btheta^\star) := \nabla^2 L(\btheta^\star) \succ 0.
\end{align}
Similar formulations have been widely used in dynamical stability analyses~\citep{wu2018sgd,cohen2021gradient,zhou2025sharpnessaware} and theoretical advances on BS scaling~\citep{mccandlish2018empirical}.

\textbf{Experimental setup.}
Our experiments are mainly conducted on \texttt{LLaMA-2} architecture~\citep{touvron2023llama} models with $93$M and $170$M parameters.
Training is performed on the \texttt{FineWeb-Edu} dataset~\citep{penedo2024the}, with sufficient training budgets ranging from $50$ to $1000$ tokens-per-parameter (TPP)\footnote{At least 10$\times$ over Chinchilla-optimal tokens~\citep{hoffmann2022training}.} and a context length of $1024$.
We adopt AdamW~\citep{kingma2014adam} with hyperparameters $\beta_1 = 0.95$, $\beta_2 = 0.95$, and weight decay $0.1$, together with gradient clipping at $1.0$ for stability.
The evaluation is conducted on a held-out validation split of approximately $50$M tokens.
More experiments on larger scales, other architectures, and optimizers are deferred to \cref{suppl:exp}.

Our experiments vary the LRs and BSs.
In \cref{sec:early}, we primarily study the role of LR and warmup length, fixing BS at $7.8$M.
In \cref{sec:late}, we focus on the effect of BS, with LR fixed at $2^{-10}$.
To decouple BS ramping from LR decay, we adopt a \emph{warmup-stable} schedule: after linear warmup to the peak value, the LR remains constant (similar to WSD~\citep{hu2024minicpm}, but without decay phase). 

\begin{figure}[tb!]
    \centering
    \includegraphics[width=1\textwidth]{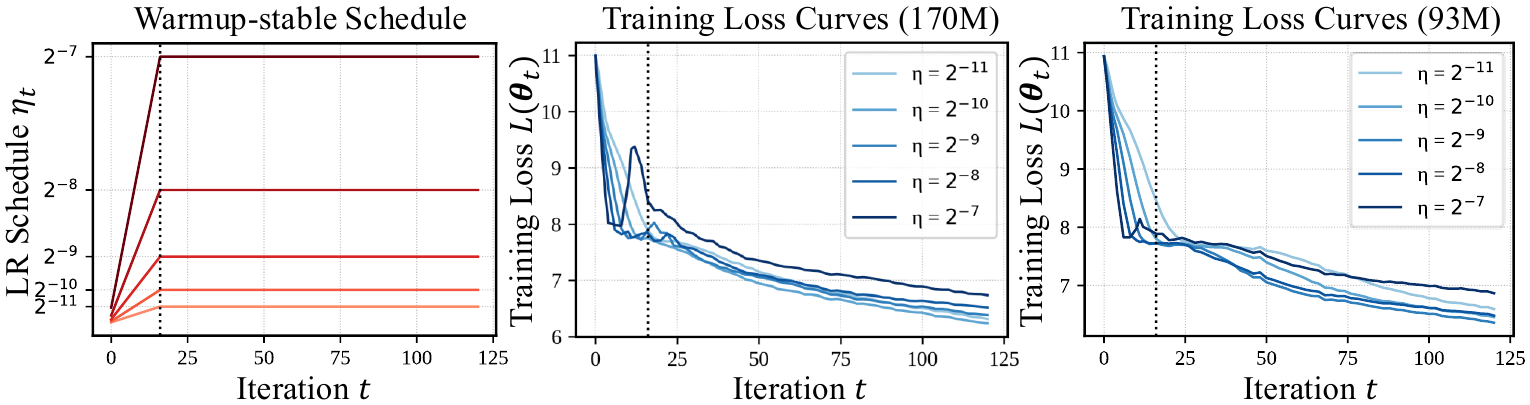}
    \caption{
    \textbf{Loss spikes and plateaus in Phase I.}
    We train a series of \texttt{LLaMA-2} models with $93$M and $170$M parameters.
    We adopt a warmup-stable schedule, where the warmup length is shortened to $16$ iterations and the peak LR is varied, $\eta \in \{2^{-11}, 2^{-10}, 2^{-9}, 2^{-8}, 2^{-7}\}$. 
    \textbf{(Left).} LR schedule: $\eta_t$ vs. training iteration $t$. 
    \textbf{(Middle, right).} Training loss curves for different model sizes: $L(\btheta_t)$ vs. training iteration $t$. 
    The vertical dashed line marks the end of the warmup phase.
    }
    \label{fig:instability and loss plateau}
    \vspace{-10pt}
\end{figure}

\section{Phase I: From Sharp to Flat Landscapes}\label{sec:early}

In this section, we provide evidence that, during Phase I, the local landscape of language models evolves from sharp regions toward flatter ones.
We first observe that training with large LRs and insufficient warmup often leads to instability and early loss plateaus. 
Via Lyapunov stability analysis, we attribute these behaviors to the sharp-to-flat dynamics.
This finding explains the necessity of LR warmup and suggests that larger peak LRs require proportionally longer warmup.

\textbf{Motivating observations: instability and loss plateaus early in training.}  
The loss curves are typically smooth initially; the model escapes from random initialization and the loss decreases rapidly.
Yet, surprisingly, when the warmup length is extremely shortened, we \emph{consistently} observe loss spikes and plateaus near the end of the warmup phase.

To demonstrate this, we train models of different sizes with a fixed warmup length of $16$ iterations while varying the peak LR.
As shown in \cref{fig:instability and loss plateau}, a loss plateau reliably appears around the end of the warmup phase across all settings.
Additionally, larger LRs produce higher spikes, which mark a characteristic feature of early training instability.
Given these results, two natural questions arise:
\begin{itemize}[leftmargin=3.5em,topsep=1pt,itemsep=1.5pt,partopsep=1.5pt, parsep=1.5pt]
    \item [\textbf{{\color{red!70!black} Q1.}}] \emph{Why does shortened warmup induce training instability?}
    \item [\textbf{{\color{red!70!black} Q2.}}] \emph{Why do spikes and plateaus occur only at the very beginning of training?}
\end{itemize}

To shed light on these questions, we analyze the dynamics of PSGD via Lyapunov stability analysis.

\textbf{Lyapunov stability analysis: sharpness matters.} 
Let $\btheta_k$, $\tilde{\btheta}_k$ be two nearby trajectories, and define their difference as $\be_k:=\tilde\btheta_k - \btheta_k$. 
When the noise term $\bxi$ is set to zero, the evolution of $\be_k$ satisfies:
\begin{align}\label{eq:linear system}
    \be_{k+1} = \be_k - \eta \mathbf{M} (\nabla L(\btheta_k + \be_k) - \nabla L(\btheta_k)) \overset{\text{(Linearization)}}{=} (\mathbf{I} - \eta \mathbf{M} \mathbf{H}(\btheta_k))\be_k,
\end{align}
The dynamics in \cref{eq:linear system} describe the local sensitivity of the iteration: if matrix $(\mathbf{I} - \eta \mathbf{M}\mathbf{H}(\btheta_k))$ repeatedly expands $\be_k$, small perturbations grow exponentially and the iterates are linearly unstable.
Intuitively, the LR $\eta$ interacts directly with the curvature of the landscape: if $\eta$ is too large relative to the sharpest direction, the update rule amplifies perturbations and leads to loss spikes.
The following lemma formalizes this stability condition for preconditioned GD.

\begin{lemma}[Stability condition for preconditioned GD]
\label{thm:stability condition}
Define the preconditioned curvature matrix $\mathbf{S}(\btheta_k) := \mathbf{M}^{1/2} \mathbf{H}(\btheta_k) \mathbf{M}^{1/2}$, and let $\{\lambda\}_{i=1}^p$ be the eigenvalues of $\mathbf{S}(\btheta_k)$. 
The linear system in \cref{eq:linear system} is asymptotically stable (i.e., $\lim_{k\to\infty} \boldsymbol{e}_k = \boldsymbol{0}$) if $\eta$ satisfies $0 < \eta < \frac{2}{\lambda_{\textup{max}}(\mathbf{S}(\btheta_k)}, \forall k \geq 0$.
\end{lemma}
\cref{thm:stability condition} shows that the Lyapunov stability is governed by the largest eigenvalue of $\mathbf{S}$.
If the curvature along the sharpest direction is too large, only a sufficiently small LR can prevent divergence.

\begin{figure}[tb!]
    \centering
    \includegraphics[width=0.78\textwidth]{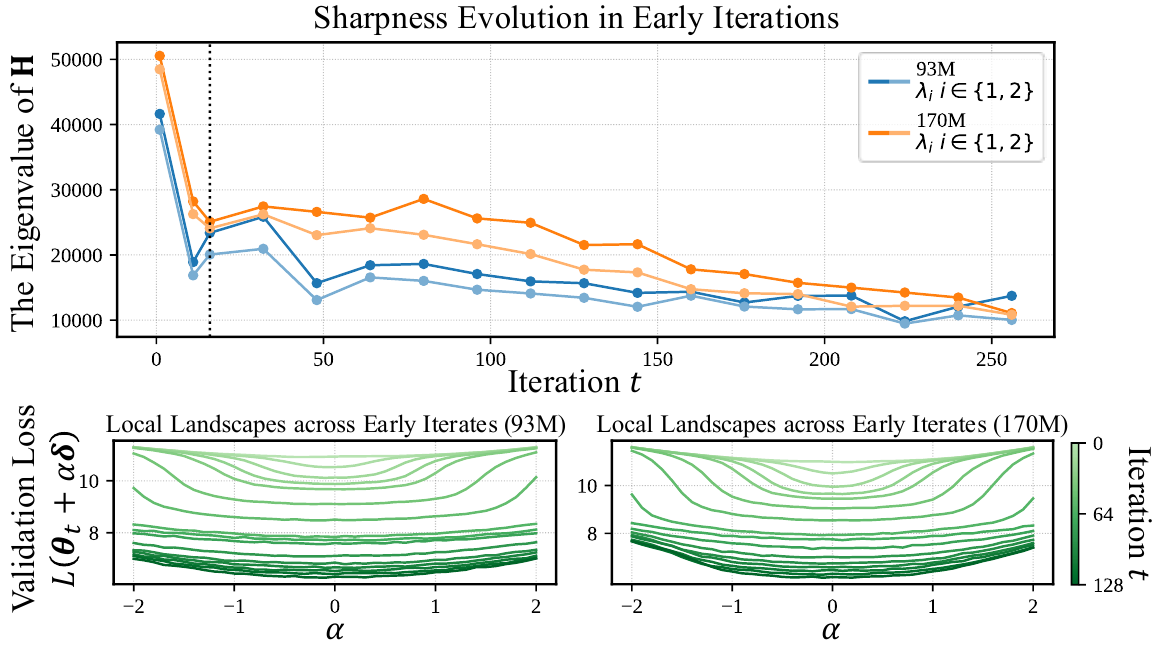}
    \caption{
    \textbf{Early pre-training shifts iterates from sharp to flat regions.}
    We visualize the local landscape geometry evolution of training runs in \cref{fig:instability and loss plateau}.
    For each model size, we select the training run with LR $2^{-10}$.
    \textbf{(Top).} Evolution of the top eigenvalues of the Hessian across iterations: $\lambda_i(\mathbf{H}(\btheta_t))$ vs. iteration $t$.
    \textbf{(Bottom).} One-dimensional loss landscape along a random perturbation direction: the perturbed loss $L(\btheta_t + \alpha \bdelta)$ vs. perturbation coefficient $\alpha$, shown across early training iterations $t$.
    }
    \label{fig:early sharpness}
\end{figure}

We next characterize the one-step loss change as $\eta$ approaches the stability boundary $2/\lambda_{\max}(\mathbf{S}_k)$.

\begin{lemma}[One-step loss change]
\label{thm:loss change}
Let $\bdelta_k := \btheta_{k+1} - \btheta_k$. Suppose that along the segment $\{\btheta_k + \alpha \bdelta_k: \alpha \in [0, 1]\}$, we have $0 \leq \lambda_{\min} (\mathbf{S}(\btheta_k + \alpha \bdelta_k)) \leq \lambda_{\max} (\mathbf{S}(\btheta_k + \alpha \bdelta_k)) \leq \Lambda_k$.
Then, \begin{align*}
    L(\btheta_{k+1}) - L(\btheta_k) \leq - \eta (1 - \frac{1}{2}\eta \Lambda_k) (\nabla L(\btheta_k))^{\top} \mathbf{M}\nabla L(\btheta_k).  
\end{align*}
In particular, if $\eta \uparrow 2/\Lambda_k$, the guaranteed decrease per step $\boxed{(L(\btheta_{k}) - L(\btheta_{k+1}))/ \eta \rightarrow 0}$.
\end{lemma}

\cref{thm:loss change} states that when $\eta$ is close to $2/\Lambda_{k}$, each update yields only a marginal decrease in loss.
Together with \cref{thm:stability condition}, it is clear that training near the stability boundary naturally leads to characteristic loss spikes and plateaus.

Importantly, the stability boundary is determined by the sharpness of the loss landscape.
To further address \textbf{\textcolor{red!70!black}{Q1-2}}, we analyze how sharpness evolves during the early phase.

\textbf{The early dynamics: from sharp to flat landscapes.}
We study how the local landscape geometry, particularly the sharpness, evolves for training runs in \cref{fig:instability and loss plateau}.
Specifically, we track the evolution of the top eigenvalues of the Hessian\footnote{Following \citet{cohen2021gradient}, we use the \texttt{Lanczos} algorithm to calculate top eigenvalues of Hessian.} $\mathbf{H}(\btheta_t)$.
For some early checkpoints $\btheta_t$, we also visualize the one-dimensional loss landscape along a random direction by plotting the function $ \mathcal{L}(\alpha) := L(\btheta_t + \alpha \bdelta)$ with $\bdelta \sim \mathcal{N}(0, \mathbf{I})$. 
\citet{NEURIPS2018_a41b3bb3} showed that these random-direction visualizations reliably capture intrinsic properties of the loss landscape properties, such as sharpness.
To ensure fair comparison across iterations, we fix the same random vector $\bdelta$ for all $\btheta_t$.

In \cref{fig:early sharpness}~(top), the largest eigenvalues of the Hessian $\mathbf{H}(\btheta_t)$ start at high values\footnote{In fact, at initialization, sharpness is extremely low but rises sharply after the first update. The sharpness curves reported in \cref{fig:early sharpness} therefore start from the first iteration.} and then decrease sharply. 
In \cref{fig:early sharpness}~(bottom), the loss landscape along a random direction progressively widens as training proceeds, confirming that the model shifts from sharp to flat regions.

\textbf{A Tuning Recipe: Larger Peak LR, Longer Warmup.}
We have seen that training stability depends on sharpness: when the landscape is shar, only a sufficiently small LR can keep training stable; 
and pre-training initially traverses from sharp landscapes to flatter ones.
Now let us return to \textbf{{\color{red!70!black} Q1}} and \textbf{{\color{red!70!black} Q2}}:
\begin{itemize}[leftmargin=3.5em,topsep=1pt,itemsep=1.5pt,partopsep=1.5pt, parsep=1.5pt]
    \item[\textbf{{\color{red!70!black} A1.}}] \emph{If the warmup phase is shortened, the LR rises too quickly while the model is still in sharp regions, leading to loss spikes and plateaus.}
    \item[\textbf{{\color{red!70!black} A2.}}] \emph{As training progresses, the landscape becomes flatter and the same LR no longer threatens stability, which explains why instability is confined to the very beginning.}
\end{itemize}

Therefore, in practice, we need a sufficiently long warmup phase to keep the LR small until sharpness has decayed, thereby preventing loss spikes and plateaus.
This further suggests a practical tuning recipe: \emph{the larger the peak LR, the longer the warmup should be}, ensuring iterates safely transition into flatter landscapes.

To validate this, we train models with varied peak LRs $\eta$ and warmup lengths $T_w$ (in iterations).
In \cref{fig:tuning recipe}, within a LR range of $2^{-8}$ to $2^{-11}$, larger peak LRs require \emph{proportionally} longer warmup to achieve the optimal validation loss $L(\btheta_{\text{best}})$.
However, this proportionality does not hold universally. 
When $\eta = 2^{-7}$, the optimal warmup length remains $2^{10}$ iterations, the same as for $\eta = 2^{-8}$.
Thus, the proportional relationship applies only within a certain range.

\section{Phase II: Local Landscape Governed by Noise Scale}\label{sec:late}

\begin{wrapfigure}{R}{0.45\textwidth}
    \vspace{-10pt}
    \centering
    \includegraphics[width=0.43\textwidth]{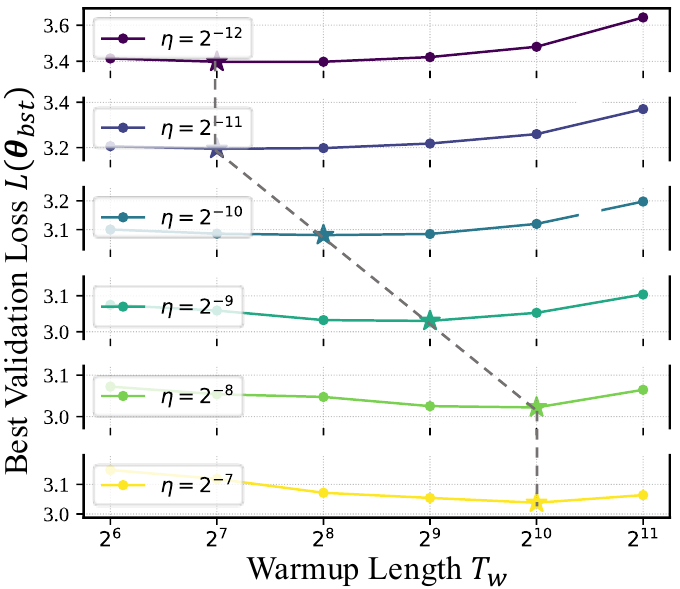}
    \caption{
    \textbf{Larger peak LR, longer warmup.}
    We plot the best validation loss $L(\btheta_{\text{bst}})$ vs. warmup $T_w$ for different $\eta$.
    Optimal $T_w$ is highlighted with a star.
    }
    \label{fig:tuning recipe}
\end{wrapfigure}

In this section, we turn to the local landscape geometry in late phase.
We observe that BS plays a central role: training with a large BS tends to find a \emph{deeper} basin of the landscape, whereas a small BS favors a \emph{wider} basin. 
Theoretically, we prove that this trade-off between \emph{widen} or \emph{deepen} is governed by the \emph{noise scale}.
Building on this, we propose a BS scheduler for the data-limited regime: \emph{use small BS early and ramp the BS late}, which consumes fewer tokens to achieve the same loss.

\textbf{The effect of BS: local landscapes late in training.}
We conduct experiments to systematically investigate the role of BS in shaping the local loss landscape  during the late phase of training.
Specifically, we train models with different BSs for $T=20{,}480$ iterations.
\cref{fig:bs effect}~(top left) shows the validation loss curves for each run.
Larger BS consistently leads to lower terminal loss and faster convergence in term of iterations\footnote{In terms of processed tokens, small BS training converges faster.}.
We then visualize the loss landscape around the final iterate $\btheta_T$.
In \cref{fig:bs effect}~(top right), it is clear that small BS produces flatter basins, whereas large BS yields deeper ones.
Furthermore, \cref{fig:bs effect}~(bottom) compares the landscape evolution of runs with $B = 0.49$M and $B = 7.8$M, indicating that in the late phase, larger BS tends to deepen the basin, while smaller BS shifts toward wider basins.

Despite these results, two key questions remain: 
\begin{itemize}[leftmargin=3.5em,topsep=1pt,itemsep=1.5pt,partopsep=1.5pt, parsep=1.5pt]
    \item[\textbf{{\color{red!70!black} Q3.}}]  \emph{Why is there a trade-off between widening and deepening the basin?}
    \item[\textbf{{\color{red!70!black} Q4.}}] \emph{Which factor underlying the hyperparameter BS governs this trade-off?}
\end{itemize} 

\begin{figure}[tb!]
    \centering
    \includegraphics[width=0.92\textwidth]{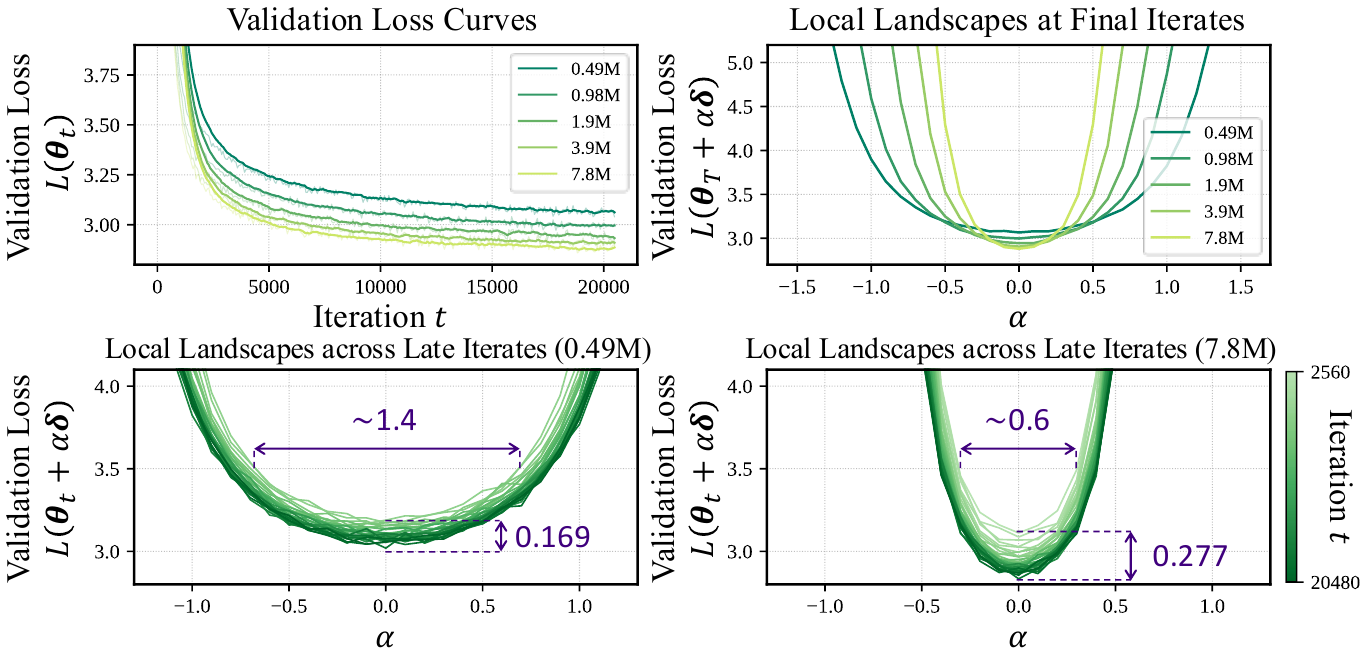}
    \caption{
    \textbf{Large BS deepens the basin, small BS widens the basin.}
    We train a series of \texttt{LLaMA-2} models (170M) for $T=20{,}480$ iterations, using BSs $B \in \{0.49\text{M}, 0.98\text{M}, 1.9\text{M}, 3.9\text{M}, 7.8\text{M}\}$.
    \textbf{(Top left).} Validation loss curves for different BSs: $L(\btheta_t)$ vs. training iteration $t$.
    \textbf{(Top right).} One-dimensional loss landscapes at the final iterates $\btheta_T$ along a random perturbation direction: perturbed loss $L(\btheta_T + \alpha \bdelta)$ vs. perturbation coefficient $\alpha$, visualized across different BSs.
    \textbf{(Bottom).} One-dimensional loss landscape: the perturbed loss $L(\btheta_t + \alpha \bdelta)$ vs. perturbation coefficient $\alpha$, shown across late training iterations $t$ for $B=0.49$M and $B=7.8$M.
    }
    \label{fig:bs effect}
    \vspace{-10pt}
\end{figure}

To delve into \textbf{{\color{red!70!black} Q3-4}}, we revisit the stochastic differential equation (SDE) in \citet{jastrzkebski2017three}.

\textbf{Widen or deepen: noise scale governs basin selection.}
Following \citet{jastrzkebski2017three}, we take the continuous-time limit of \cref{eq:noisy dynamics}.
Suppose that the noise covariance satisfies\footnote{The assumption of noise covariance is justified in \cref{sec:stochastic-analysis}.}
$\frac{\eta}{B} \mathbf{M} \mathbf{\Sigma}(\btheta^\star) \mathbf{M}^\top = 2\tau \mathbf{M} + \cO(\eta)$ for some temperature $\tau > 0$. 
As $\eta \to 0$, the scaled discrete process $\btheta_{\lfloor t/\eta \rfloor}$ converges weakly to the It\^o{} SDE:
\begin{align}\label{eq: SDE_main}
d\btheta_t &= -\mathbf{M} \nabla L(\btheta_t) dt + \sqrt{2\tau} \mathbf{M}^{1/2} dW_t
\end{align}
where $W_t$ is standard Brownian motion and the noise scale $\tau$ is proportional to $\eta / B$.

Building on \cref{eq: SDE_main} and the local quadratic model in \cref{eq:quadratic approximation}, we establish the trade-off between deepening and widening the loss basin.

\begin{theorem}[Depth-flatness trade-off]\label{thm:trade-off}
Let the empirical risk $L(\btheta)$ admit multiple local minima $\{\btheta_i^\star\}_{i=1}^m$ with Hessians $\mathbf{H}(\btheta_i^\star) \succ 0$.
Under the SDE in \cref{eq: SDE_main} with temperature $\tau$, the stationary probability that training resides in basin $i$ is given by:
\begin{align*}
P_\tau(\text{basin } i) &= \frac{\exp(-F_i(\tau)/\tau)}{\sum_j \exp(-F_j(\tau)/\tau)}, \quad 
F_i(\tau) := \colorbox{cyan!20}{$L(\btheta_i^\star)$} + \frac{\colorbox{green!8}{$\tau$}}{2} \colorbox{orange!20}{$\log\det \mathbf{H}(\btheta_i^\star)$}. 
\end{align*}
\end{theorem}

\cref{thm:trade-off} states that the basin selection is controlled by the free energy function $F(\tau) = L(\btheta^\star) + \frac{\tau}{2} \log\det \mathbf{H}(\btheta^\star)$.
In early training, the loss term $L(\btheta^\star)$ dominates, so the model primarily seeks regions of lower loss.
In later training, $L(\btheta^\star)$ is comparable to the flatness penalty $\log\det \mathbf{H}(\btheta^\star)$, and basin selection becomes increasingly sensitive to the noise scale $\tau \propto \eta/B$.

\textbf{Efficient pre-training: a BS scheduler in data-limited regime.}
Turning back to \textbf{{\color{red!70!black} Q3}} and \textbf{{\color{red!70!black} Q4}}, the trade-off arises because basin selection balances loss minimization against curvature regularization (\textbf{{\color{red!70!black} A3}}), with the governing factor being the noise scale $\tau$ (\textbf{{\color{red!70!black} A4}}).
Since the primary objective of pre-training is to minimize the training loss\footnote{Note that our analysis focuses on how reduced noise (e.g., via larger BS) helps the optimizer move into deeper minima, conceptually different from the flat-minima perspective in fully interpolating regimes.}, this balance naturally favors largest BS available (small $\tau$).
In practice, however, data availability is limited, and excessively large BS substantially increase data consumption\footnote{For example, in \cref{fig:bs effect}~(top right), when $B = 7.8$M, the run consumes approximately $160$B tokens.}.
Thus, scheduling BS in pre-training is crucial, particularly in the data-limited regime.

\textbullet~\textbf{BS scheduler: design principle I.}
Inspired by our theory, loss reduction dominates early in training, during which large BS yields limited benefit. 
This suggests the first design principle.

\begin{snugshade}
\begin{center}
\vspace*{.1em}
{\bf Design principle I.} {\em Start the training process with a small BS before increasing it later.}
\end{center}
\vspace*{-.45em}
\end{snugshade}
\vspace{-.7em}

Related ideas were noted by \citet{li2025minimax,merrill2025critical}, often referred to as \emph{BS warmup}. 
However, a key difference in our design lies in when the BS should be increased.
Surprisingly, we find that ramping BS \emph{later in training} yields consistently \emph{greater} performance.

\begin{figure}[tb!]
    \centering
    \includegraphics[width=1\textwidth]{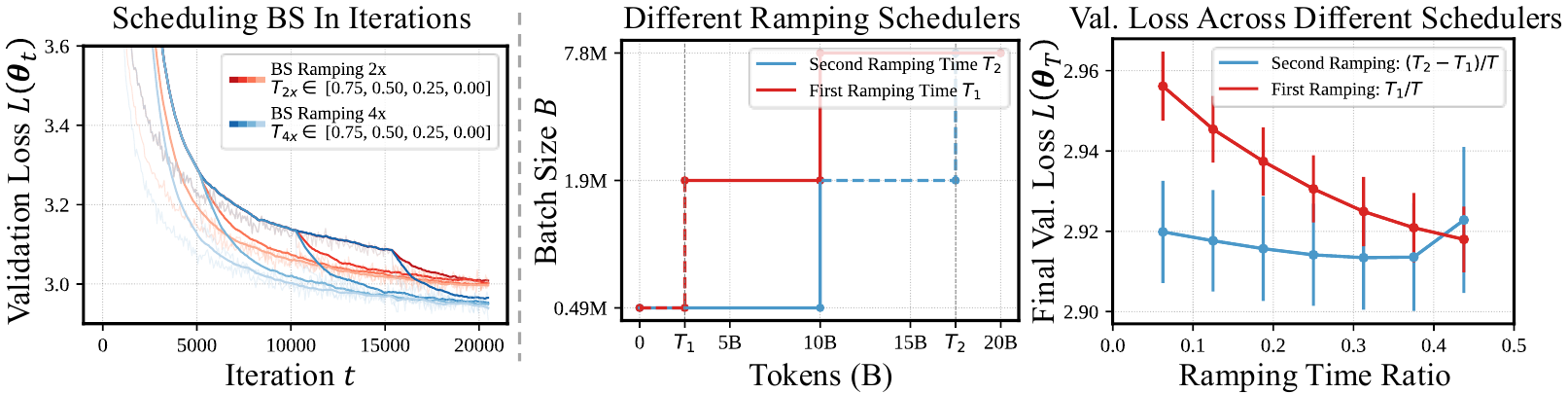}
    \caption{
    \textbf{(Left) Collapse of loss curves under different BS schedules.}
    Validation loss curves for training with different BS scheduling.
    In all runs, BS starts at $0.49$M.
    For blue curves, BS is ramped up to $4\times$ its initial value; for red curves, BS is ramped to $2\times$.
    The ramping times, $T_{2\times}$ or $T_{4\times}$, are varied across different positions.
    \textbf{(Middle, Right) Ramping BS is more efficient late in training.} 
    We evaluate a two-stage BS-ramping schedule with ramp times $T_1$ and $T_2$.
    For the red curves, we fix $T_2=10$B and vary $T_1$; for the blue curves, we fix $T_1=10$B and vary $T_2$.
    \textbf{(Middle).} Illustration of BS schedulers.
    \textbf{(Right).} Final validation loss vs. the relative ramping time, i.e., $(T_1)/T, (T_2 - T_1)/T\in [0,0.5]$, where $T$ denotes the total training tokens.
    }
    \label{fig:timing}
\end{figure}

\textbullet~\textbf{BS scheduler: design principle II.}
To study this, we train models with different BS schedulers while keeping total training iterations fixed.
In \cref{fig:timing}~(left), all runs begin with an initial BS of $0.49$M and ramp up to either $4\times$ or $2\times$ that value at different training iterations.
Remarkably, all loss curves eventually collapse onto the same trajectory, regardless of when the BS ramping occurs.
Note that when measured at the same training iteration, ramping the BS earlier results in higher data consumption.
This indicates that early BS ramping offers no efficiency advantage, achieving the same loss but consuming more data.

We next evaluate BS schedules under a fixed token budget.
Specifically, we consider a two-stage BS-ramping scheduler characterized by ramp times $T_1$ and $T_2$.
To isolate the effect of each stage, we vary either $T_1$ or $T_2$ while keeping the other fixed.
See \cref{fig:timing}~(middle) for an illustration.
In \cref{fig:timing}~(right), a clear trend emerges: BS ramping is most effective when applied late in training (i.e., with $T_1$ and $T_2$ large), whereas ramping too early consistently harms final performance.

Together, since BS ramping ultimately leads all runs onto the same trajectory, delaying it allows maximal progress (lower loss) along that trajectory under a data-limited budget.
This behavior also aligns with our theory.
In the late phase of training, the flatness penalty becomes comparable to the loss term, and BS ramping sharply reduces the noise scale, driving rapid convergence toward deeper minima.
This consistency between theory and practice leads to our second principle.

\begin{snugshade}
\begin{center}
\vspace*{.1em}
{\bf Design principle II.} {\em Ramp the batch size late in training—when loss reduction becomes marginal.}
\end{center}
\vspace*{-.45em}
\end{snugshade}
\vspace{-.7em}

\begin{figure}[tb!]
    \centering
    \includegraphics[width=1\textwidth]{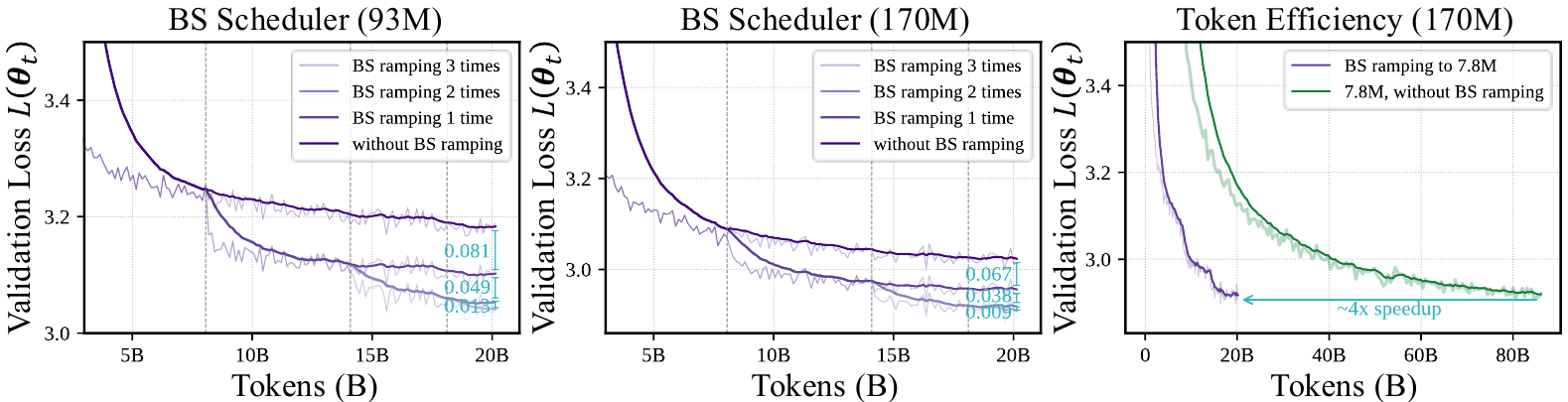}
    \caption{
    \textbf{BS scheduling improves data efficiency.}
    We train \texttt{LLaMA-2} models with $93$M and $170$M parameters, using a BS schedule that starts at $0.49$M and increases by $4\times$ at each ramp. 
    Models are trained with $1, 2$, or $3$ ramping steps, while models without ramping serve as the baseline. 
    Vertical gray dashed lines indicate ramping positions.
    \textbf{(Left, Middle).} The validation curves for each run.
    \textbf{(Right).} Comparison between training with BS ramping to $7.8$M and training with a fixed $7.8$M BS.
    }
    \label{fig:bs scheduler}
\end{figure}

To further validate our design principle, we train models using a BS schedule that starts at $0.49$M and ramps by $4\times$ whenever loss minimization slows. 
In \cref{fig:bs scheduler}~(left, middle), models with $1,2$ or $3$ BS ramping steps achieves significant lower validation loss. 
While additional ramping steps provide diminishing returns, each step still offers a measurable improvement.
Moreover, \cref{fig:bs scheduler}~(right) highlights the data-efficiency of the BS scheduling: ramping the BS up to $7.8$M achieves nearly the same final validation loss as training with a fixed $7.8$M BS, but requires only about $\frac{1}{4}$ of the tokens (i.e., a $\sim4\times$ speedup). 
These results confirm that our BS scheduling design preserves the benefits of large BS while substantially reducing data consumption.

\textbf{Comparison with} \citet{mccandlish2018empirical,merrill2025critical}.
\citet{mccandlish2018empirical} linked BS scaling to the gradient noise and introduced the notion of CBS.
\citet{merrill2025critical} explored the BS scheduling (BS warmup), doubling BS once the CBS exceeds the current BS.
We differs from CBS-related works: without estimating CBS on the fly, we propose to ramp the BS late in training.

\begin{figure}[tb!]
    \centering
    \includegraphics[width=0.88\textwidth]{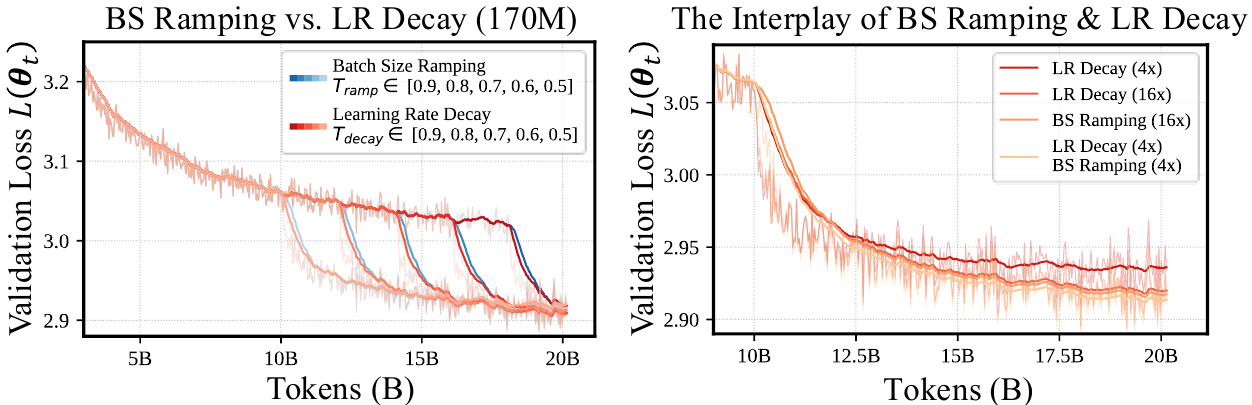}
    \caption{
    \textbf{(Left) BS ramping performs similarly to LR decay.}
    Validation loss curves for training with either BS ramping or LR decay.
    For BS ramping, BS increases to $16\times$ its initial value; for LR decay, the LR drops to $1/16$ of its initial value.
    Each method applies a single step at varying positions.
    \textbf{(Right) Interplay between BS ramping and LR decay.}
    We evaluate four different scheduling strategies.
    At the 10B tokens, (a) LR drops to $1/4$ of its initial value;
    (b) LR drops to $1/16$;
    (c) BS ramps to $16\times$.
    (d) LR drops to $1/4$ and BS ramps to $4\times$.
    In all runs (both left and right), BS starts at $0.49$M and LR begins at $2^{-10}$ (after linear warmup).
    }
    \label{fig:discussion}
\end{figure}

\section{More Discussions: LR Decay and BS Ramping}
We have excluded LR decay in our experiments to isolate the effect of BS ramping. 
Yet, recall that the noise scale $\tau$ is proportional to $\eta / B$. 
Our theory suggests that decaying the LR and ramping the BS both reduce the noise scale, and thus may have similar effects on basin selection. 

\textbf{Comparing BS ramping with LR decay.}
We first train models using either BS ramping or LR decay. 
Both methods apply a one-time step change: BS ramping multiplies the BS by $16$ at $T_{\text{ramp}}$, while LR decay divides the LR by $16$ at $T_{\text{decay}}$.
We align $T_{\text{ramp}}$ and $T_{\text{decay}}$ so that the changes occur at the same positions.
In \cref{fig:discussion}~(left), BS ramping and LR decay produce remarkably similar validation loss curves across all change positions, consistent with the idea that both reduce the noise scale in comparable ways.

\textbf{Interacting BS ramping with LR decay.}
Furthermore, we study the combined effect of using both BS ramping and LR decay.
Specifically, we decay the LR by $4\times$ and simultaneously ramp the BS by $4\times$ at 10B tokens.
We compare this hybrid schedule with three baselines: at the same point, we (a) drops the LR by $4\times$, (b) drops the LR by $16\times$ and (c) ramps the BS by $16\times$.
We denote the hybrid schedule by (d).
In \cref{fig:discussion}~(right), three of the schedules (b c d) produce nearly identical loss curves.
Crucially, these three configurations yield the same noise scale $\eta / B$.
In contrast, schedule (a) results in a noticeably different trajectory.

In summary, our results confirm our theoretical prediction: training dynamics in the late phase are governed by the noise scale $\eta / B$.
LR decay reduce the noise scale in the same manner as BS ramping, and any scheduler that preserves $\eta/B$ will exhibit nearly identical behavior.

\section{Conclusion and Limitations}

In conclusion, we present a study of how local landscape geometry evolves during language model pre-training. 
Our analysis reveals two phases: an early sharp-to-flat transition and a late noise-governed regime.
Phase I explains the necessity of LR warmup, suggesting that larger peak LRs require proportionally longer warmup lengths. 
Phase II motivates a BS scheduling that starts with small BS and increases the BS late in training.
\textbf{Limitations.}
The current theory primarily relies on strong assumptions, such as infinite-small LR in SDE.
A natural future direction is to generalize the theory to more realistic settings.
Additionally, the current theory cannot fully explain the collapse of loss curves under different BS schedules.
Understanding the learning dynamics under different BS schedules remains an open question.

\bibliography{iclr2027_conference}
\bibliographystyle{iclr2027_conference}

\clearpage
\newpage

\appendix

\section{Terminologies}

\begin{table}[h]
    \centering
    \renewcommand{\arraystretch}{1.25}
    \begin{tabular}{p{0.18\textwidth} p{0.34\textwidth} p{0.38\textwidth}}
        \hline
        \textbf{Terminology} & \textbf{General Meaning} & \textbf{Usage in This Paper} \\
        \hline
        Sharpness
        &
        A measure of curvature in the loss landscape, often characterized via the Hessian. Different works may define it differently.
        &
        We define sharpness as the curvature along the sharpest direction of the loss landscape. Mathematically, it is presented as the largest eigenvalue of the Hessian $\lambda_{\max}(\mathbf{H}(\boldsymbol{\theta}_t))$ or of the preconditioned curvature matrix $\lambda_{\max}(\mathbf{S}(\boldsymbol{\theta}_t))$.
        \\
        \hline
        Flat/sharp minimum
        &
        A minimum is a point where the gradient vanishes $\nabla L(\boldsymbol{\theta}) = 0$ and the loss does not decrease in a small neighborhood. A sharp minimum has large curvature; a flat minimum has small curvature.
        &
        We use these terms sparingly and follow the standard definitions from the sharpness/flat-minima literature.
        \\
        \hline
        Wide/deep basin
        &
        A loss basin is a region of the landscape surrounding a minimum. A wide basin rises loss slowly in most directions, whereas a deep basin has a significantly lower minimum value compared to its surroundings.
        &
        We use these terms to establish the depth–flatness trade-off: large noise scales tend to find wide basins, while small noise scales tend to find deeper regions with lower loss.
        \\
        \hline
    \end{tabular}
    \caption{Terminology and usage in this paper.}
    \label{tab:term}
\end{table}

\clearpage
\newpage

\section{Missing Proof}
\label{suppl:proof}

\subsection{Phase I: Lyapunov Stability Analysis}
\label{sec:stability-analysis}

\begin{lemma}[Stability Condition for Preconditioned GD]
\label{thm:main}

Define the preconditioned curvature matrix $\mathbf{S}(\btheta_k) := \mathbf{M}^{1/2} \mathbf{H}(\btheta_k) \mathbf{M}^{1/2}$, and let $\{\lambda\}_{i=1}^p$ be the eigenvalues of $\mathbf{S}(\btheta_k)$. 
The linear system in \cref{eq:linear system} is asymptotically stable (i.e., $\lim_{k\to\infty} \boldsymbol{e}_k = \boldsymbol{0}$) if $\eta$ satisfies 
\begin{equation}
0 < \eta < \frac{2}{\lambda_{\textup{max}}(\mathbf{S}(\btheta_k)}, \forall k \geq 0.
\label{eq:eta_condition}
\end{equation}
\end{lemma}

\begin{proof}
Since $\be_{k+1} = (\mathbf{I} - \eta \mathbf{M} \mathbf{H}_k) \be_k$, the linear system is asymptotically stable if all eigenvalues of $\mathbf{I} - \eta \mathbf{M} \mathbf{H}_k$ have magnitude less than 1. Note that:
\begin{equation}
    \mathbf{I} - \eta \mathbf{M} \mathbf{H}_k = \mathbf{M}^{1/2} (\mathbf{I} - \eta \mathbf{S}_k) \mathbf{M}^{-1/2},
\end{equation}
so the eigenvalues are $1 - \eta \lambda_j(\mathbf{S}_k)$. The stability condition $|1 - \eta \lambda_j| < 1$ for all $j$ is equivalent to:
\begin{equation}
    0 < \eta < \frac{2}{\lambda_{\max}(\mathbf{S}_k)}.
\end{equation}
\end{proof}

\begin{lemma}[Exact one-step loss change]\label{lem:exact_loss_change}
Define:
\begin{align*}
\mathbf{S}(\btheta) &:= \mathbf{M}^{1/2} \mathbf{H}(\btheta) \mathbf{M}^{1/2}, \\
\bg_k &:= \mathbf{M}^{1/2} \nabla L(\btheta_k), \\
\bdelta_k &:= \btheta_{k+1} - \btheta_k = -\eta \mathbf{M} \nabla L(\btheta_k).
\end{align*}
Then the true loss change can be written exactly as
\begin{equation}
L(\btheta_{k+1}) - L(\btheta_k) = 
-\eta \|\nabla L(\btheta_k)\|_{\mathbf{M}}^2 
+ \eta^2 \int_0^1 (1-t) \, \bg_k^\top \mathbf{S}(\btheta_k + t\bdelta_k) \bg_k \, dt,
\end{equation}
where $\|\nabla L(\btheta_k)\|_{\mathbf{M}}^2 := { ( \nabla L(\btheta_k) ) }^\top \mathbf{M} \nabla L(\btheta_k)$.
\end{lemma}

\begin{proof}
Let $\bdelta_k := \btheta_{k+1} - \btheta_k = -\eta \mathbf{M} \nabla L(\btheta_k)$ and define the scalar function
\[
\phi(t) := L(\btheta_k + t\bdelta_k), \quad t \in [0,1].
\]
Then
\[
L(\btheta_{k+1}) - L(\btheta_k) = \phi(1) - \phi(0).
\]
Compute the derivatives:
\begin{align*}
\phi'(t) &= { ( \nabla L(\btheta_k + t\bdelta_k) ) }^\top \bdelta_k, \\
\phi''(t) &= \bdelta_k^\top \mathbf{H}(\btheta_k + t\bdelta_k) \bdelta_k.
\end{align*}
By Taylor's theorem with integral remainder:
\[
\phi(1) - \phi(0) = \phi'(0) + \int_0^1 (1-t) \phi''(t) \, dt.
\]
Now evaluate at $t=0$:
\[
\phi'(0) = { ( \nabla L(\btheta_k) ) }^\top \bdelta_k = -\eta { ( \nabla L(\btheta_k) ) }^\top \mathbf{M} \nabla L(\btheta_k) = -\eta \|\nabla L(\btheta_k)\|_{\mathbf{M}}^2.
\]
For the second derivative term:
\[
\phi''(t) = \bdelta_k^\top \mathbf{H}(\btheta_k + t\bdelta_k) \bdelta_k = \eta^2 \bg_k^\top \mathbf{S}(\btheta_k + t\bdelta_k) \bg_k,
\]
since $\bdelta_k = -\eta \mathbf{M} \nabla L(\btheta_k)$ and $\bg_k = \mathbf{M}^{1/2} \nabla L(\btheta_k)$, and thus
\[
\bdelta_k^\top \mathbf{H}(\cdot) \bdelta_k = \eta^2 \bg_k^\top \mathbf{S}(\cdot) \bg_k.
\]
Substituting both terms yields the result.
\end{proof}

\begin{lemma}[One-step Loss Change]\label{lem:loss_change}
Let $\bdelta_k := \btheta_{k+1} - \btheta_k$. Suppose that along the segment $\{\btheta_k + \alpha \bdelta_k: \alpha \in [0, 1]\}$, we have $0 \leq \lambda_{\min} (\mathbf{S}(\btheta_k + \alpha \bdelta_k)) \leq \lambda_{\max} (\mathbf{S}(\btheta_k + \alpha \bdelta_k)) \leq \Lambda_k$.
Then, \begin{align*}
    L(\btheta_{k+1}) - L(\btheta_k) \leq - \eta (1 - \frac{1}{2}\eta \Lambda_k) (\nabla L(\btheta_k))^{\top} \mathbf{M}\nabla L(\btheta_k).  
\end{align*}

In particular, if $\eta \leq 2/\Lambda_k$, each update is guaranteed to non-increasing in loss, i.e., $L(\btheta_{k+1}) \leq L(\btheta_k)$. 
Instead, if $\eta \uparrow 2/\Lambda_k$, the guaranteed decrease per step $\boxed{(L(\btheta_{k}) - L(\btheta_{k+1}))/ \eta \rightarrow 0}$.
\end{lemma}

\begin{proof}
From Lemma~\ref{lem:exact_loss_change}, we have
\[
\bg_k^\top \mathbf{S}(\btheta_k + t\bdelta_k) \bg_k \leq \Lambda_k \|\bg_k\|^2
\]
for all $t$, since $\mathbf{S}(\cdot)$ is symmetric. Therefore,
\begin{align*}
L(\btheta_{k+1}) - L(\btheta_k) 
&\leq -\eta \|\nabla L(\btheta_k)\|_{\mathbf{M}}^2 + \eta^2 \Lambda_k \|\bg_k\|^2 \int_0^1 (1-t) \, dt \\
&= -\eta \|\nabla L(\btheta_k)\|_{\mathbf{M}}^2 + \tfrac{1}{2} \eta^2 \Lambda_k \|\bg_k\|^2.
\end{align*}
Note that $\|\bg_k\|^2 = \|\nabla L(\btheta_k)\|_{\mathbf{M}}^2$, yielding the result.
\end{proof}

\subsection{Phase II: SDE Analysis}
\label{sec:stochastic-analysis}

\subsubsection{Discrete-Time Solution}

\begin{lemma}[Eigenbasis Decomposition]
\label{lem: Eigenbasis Decomposition}
Let $\mathbf{S} := \mathbf{M}^{1/2} \mathbf{H}(\btheta^\star) \mathbf{M}^{1/2}$ with eigendecomposition $\mathbf{S} = \mathbf{Q}\boldsymbol{\Lambda} \mathbf{Q}^\top$, $\boldsymbol{\Lambda} = \text{diag}(\lambda_1, \dots, \lambda_d)$. Define $\mathbf{G} := \mathbf{Q}^\top \mathbf{M}^{1/2} \boldsymbol{\Sigma}(\btheta^\star)  \mathbf{M}^{1/2} \mathbf{Q}/B$. In coordinates $\bw_k := \mathbf{Q}^\top \mathbf{M}^{-1/2} \be_k$, the recursion gives:
\begin{align*}
\bw_{k+1} &= (\mathbf{I} - \eta \boldsymbol{\Lambda}) \bw_k + \eta \boldsymbol{\zeta}_k, \quad \mathbb{E}[\boldsymbol{\zeta}_k\boldsymbol{\zeta}_k^\top] = \mathbf{G}
\end{align*}
The stationary covariance $\boldsymbol{\Sigma}_w$ has diagonal elements:
\begin{align}
{(\boldsymbol{\Sigma}_w)}_{jj} &= \frac{\eta^2 \mathbf{G}_{jj}}{1 - {(1 - \eta\lambda_j)}^2} = \frac{\eta \mathbf{G}_{jj}}{2\lambda_j - \eta\lambda_j^2}
\end{align}
\end{lemma}

\begin{proof}
First, we verify that $\mathbf{S} = \mathbf{M}^{1/2} \mathbf{H}(\btheta^\star) \mathbf{M}^{1/2}$ can be eigendecomposed. Since both $\mathbf{M}$ and $\mathbf{H}(\btheta^\star)$ are positive definite matrices, $\mathbf{S}$ is also positive definite matrix. By the spectral theorem, $\mathbf{S}$ admits the eigendecomposition $\mathbf{S} = \mathbf{Q}\boldsymbol{\Lambda} \mathbf{Q}^\top$, where $\mathbf{Q}$ is orthogonal and $\boldsymbol{\Lambda} = \text{diag}(\lambda_1, \dots, \lambda_d)$ with $\lambda_i > 0$.

Starting from $\be_{k+1} = \mathbf{A}\be_k + \eta \mathbf{M}\bxi_k$ with $\mathbf{A} = \mathbf{I} - \eta \mathbf{M}\mathbf{H}(\btheta^\star)$, we change variables to $\bw_k = \mathbf{Q}^\top \mathbf{M}^{-1/2} \be_k$.
\begin{align*}
\bw_{k+1} &= \mathbf{Q}^\top \mathbf{M}^{-1/2} \be_{k+1} = \mathbf{Q}^\top \mathbf{M}^{-1/2} (\mathbf{A}\be_k + \eta \mathbf{M}\bxi_k) \\
&= \mathbf{Q}^\top \mathbf{M}^{-1/2} (\mathbf{I} - \eta \mathbf{M}\mathbf{H}(\btheta^\star)) \be_k + \eta \mathbf{Q}^\top \mathbf{M}^{1/2} \bxi_k \\
&= \mathbf{Q}^\top \mathbf{M}^{-1/2} \be_k - \eta \mathbf{Q}^\top \mathbf{M}^{1/2} \mathbf{H}(\btheta^\star) \be_k + \eta \mathbf{Q}^\top \mathbf{M}^{1/2} \bxi_k \\
&= \bw_k - \eta \mathbf{Q}^\top \mathbf{M}^{1/2} \mathbf{H}(\btheta^\star) \mathbf{M}^{1/2} \mathbf{Q} \bw_k + \eta \mathbf{Q}^\top \mathbf{M}^{1/2} \bxi_k \\
&= \bw_k - \eta \mathbf{Q}^\top \mathbf{S} \mathbf{Q} \bw_k + \eta \mathbf{Q}^\top \mathbf{M}^{1/2} \bxi_k \\
&= (\mathbf{I} - \eta \boldsymbol{\Lambda}) \bw_k + \eta \mathbf{Q}^\top \mathbf{M}^{1/2} \bxi_k
\end{align*}

Defining $\boldsymbol{\zeta}_k := \mathbf{Q}^\top \mathbf{M}^{1/2} \bxi_k$, we get:
\begin{align*}
\mathbb{E}[\boldsymbol{\zeta}_k\boldsymbol{\zeta}_k^\top] &= \mathbf{Q}^\top \mathbf{M}^{1/2} \mathbb{E}[\bxi_k\bxi_k^\top] \mathbf{M}^{1/2} \mathbf{Q} = \mathbf{Q}^\top \mathbf{M}^{1/2} \boldsymbol{\Sigma}(\btheta^\star)  \mathbf{M}^{1/2} \mathbf{Q}/B =: \mathbf{G}
\end{align*}

As matrix $(\mathbf{I} - \eta \boldsymbol{\Lambda})$ is diagonal, the recursion now decouples into independent scalar equations for each component $j$:
\begin{align*}
(\bw_{k+1})_j &= (1-\eta\lambda_j)(\bw_k)_j + \eta(\boldsymbol{\zeta}_k)_j.
\end{align*}

For each component $j$, the stationary variance satisfies:
\begin{align}
(\boldsymbol{\Sigma}_w)_{jj} &= (1-\eta\lambda_j)^2 (\boldsymbol{\Sigma}_w)_{jj} + \eta^2 \mathbf{G}_{jj} \label{eq:line4}
\end{align}

Solving for $(\boldsymbol{\Sigma}_w)_{jj}$:
\begin{align*}
(\boldsymbol{\Sigma}_w)_{jj} &= \frac{\eta^2 \mathbf{G}_{jj}}{1 - (1-\eta\lambda_j)^2} = \frac{\eta^2 \mathbf{G}_{jj}}{1 - (1-2\eta\lambda_j + \eta^2\lambda_j^2)} \\
&= \frac{\eta^2 \mathbf{G}_{jj}}{2\eta\lambda_j - \eta^2\lambda_j^2} = \frac{\eta \mathbf{G}_{jj}}{2\lambda_j - \eta\lambda_j^2}
\end{align*}
\end{proof}

\subsubsection{Continuous-Time Limit}
We now take the continuous-time limit ($\eta \to 0$) to derive a simpler universal theory.
The exact solution for the variance in the eigenbasis from \cref{lem: Eigenbasis Decomposition}, i.e., $(\boldsymbol{\Sigma}_w)_{jj} = \eta \mathbf{G}_{jj}/(2\lambda_j - \eta\lambda_j^2)$, guides the necessary scaling for the continuous-time limit.
Because $(\boldsymbol{\Sigma}_w)_{jj}$ converges to a finite non-zero value as $\eta \to 0$, the numerator $\eta \mathbf{G}_{jj}$ must remain finite. 
This suggests defining a quantity $\tau $ such that for each mode $j$:
\[
\eta \, \mathbf{G}_{jj} \to 2\tau \quad \text{as} \quad \eta \to 0.
\]
We strengthen this to :
\[
\eta \, \mathbf{G} \to 2\tau\mathbf{I} \quad \text{as} \quad \eta \to 0.
\]
Recalling that $\mathbf{G} = \mathbf{Q}^\top \mathbf{M}^{1/2} \boldsymbol{\Sigma}(\btheta^\star)  \mathbf{M}^{1/2} \mathbf{Q}/B$, this condition in the original coordinate system translates to the required scaling for the noise covariance:
\[
\frac{\eta}{B} \mathbf{M} \boldsymbol{\Sigma}(\btheta^\star) \mathbf{M}^{\top} \to 2\tau \mathbf{M}.
\]

\begin{proposition}[Convergence to SDE]\label{thm: converge to SDE}
Consider the scaled discrete process $\btheta_{\lfloor t/\eta \rfloor}$ as $\eta \to 0$.
Suppose the noise covariance satisfies
\begin{align}\label{eq:scaling}
\frac{\eta}{B} \mathbf{M} \mathbf{\Sigma}(\btheta^\star) \mathbf{M}^\top &= 2\tau \mathbf{M} + O(\eta),  
\end{align}
for some temperature $\tau > 0$. Then the process converges weakly to the It\^o{} SDE:
\begin{align}\label{eq:SDE_main}
d\btheta_t &= -\mathbf{M} \nabla L(\btheta_t) dt + \sqrt{2\tau} \mathbf{M}^{1/2} dW_t 
\end{align}
where $W_t$ is standard Brownian motion.
\end{proposition}
\begin{proof}
Consider the discrete preconditioned SGD update:
\[
\btheta_{k+1} = \btheta_k - \eta \mathbf{M} (\nabla L(\btheta_k) + \bxi_k),
\]
Define the scaled process \(\btheta^{(\eta)}(t) = \btheta_{\lfloor t/\eta \rfloor}\). The increment \(\Delta \btheta_k = \btheta_{k+1} - \btheta_k\) satisfies:
\[
\mathbb{E}[\Delta \btheta_k \mid \btheta_k = \btheta] = -\eta \mathbf{M} \nabla L(\btheta),
\]
\[
\text{Cov}(\Delta \btheta_k \mid \btheta_k = \btheta) = \frac{\eta^2}{B} \mathbf{M} \boldsymbol{\Sigma}(\btheta^\star)\mathbf{M}^\top .
\]
Given the scaling condition \cref{eq:scaling}, the covariance is \(O(\eta)\).

The generator \(\mathcal{L}^{(\eta)}\) of the discrete process for a smooth function $\bff$ is:
\[
\mathcal{L}^{(\eta)} \bff(\btheta) = \frac{1}{\eta} \mathbb{E}[\bff(\btheta_{k+1}) - \bff(\btheta_k) \mid \btheta_k = \btheta].
\]
Using a Taylor expansion and taking conditional expectation:
\[
\mathbb{E}[\bff(\btheta_{k+1}) - \bff(\btheta_k) \mid \btheta] = -\eta \nabla \bff(\btheta)^\top \mathbf{M} \nabla L(\btheta) + \frac{1}{2} \mathbb{E}[(\Delta \btheta)^\top \nabla^2 \bff(\btheta) \Delta \btheta] + O(\eta^{3/2}).
\]
For the second term, with $\Delta \btheta = -\eta M(\nabla L(\btheta) + \bxi_k)$:
\begin{align*}
\mathbb{E}[(\Delta \btheta)^\top \nabla^2 \bff(\btheta) \Delta \btheta] &= \eta^2 \mathbb{E}[(\nabla L(\btheta) + \bxi_k)^\top \mathbf{M}^\top \nabla^2 \bff(\btheta) \mathbf{M} (\nabla L(\btheta) + \bxi_k)] \\
&= \eta^2 \mathbb{E}[\bxi_k^\top \mathbf{M}^\top \nabla^2 \bff(\btheta) \mathbf{M} \bxi_k] + O(\eta^2) \\
&= \eta^2 \text{Tr}(\mathbf{M}^\top \nabla^2 \bff(\btheta) \mathbf{M} \mathbb{E}[\bxi_k \bxi_k^\top]) + O(\eta^2) \\
&= \frac{\eta^2}{B} \text{Tr}(\mathbf{M}^\top \nabla^2 \bff(\btheta) \mathbf{M} \mathbf{\Sigma}(\btheta^\star)) + O(\eta^2) \\
&= \frac{\eta^2}{B} \text{Tr}(\mathbf{M} \mathbf{\Sigma}(\btheta^\star) \mathbf{M}^\top \nabla^2 \bff(\btheta)) + O(\eta^2)
\end{align*}
where we used $\mathbb{E}[\bxi^\top \mathbf{A} \bxi] = \text{Tr}(A \mathbb{E}[\bxi \bxi^\top])$ and trace cyclicity $\text{Tr}(\mathbf{ABC}) = \text{Tr}(\mathbf{CAB})$.

Therefore:
\[
\frac{1}{2} \mathbb{E}[(\Delta \btheta)^\top \nabla^2 \bff(\btheta) \Delta \btheta] = \frac{\eta^2}{2B} \text{Tr}(\mathbf{M} \mathbf{\Sigma}(\btheta^\star) \mathbf{M}^\top \nabla^2 \bff(\btheta)) + O(\eta^2).
\]

Using the scaling condition \cref{eq:scaling}, we have :
\[
\frac{\eta^2}{2B} \text{Tr}(\mathbf{M} \boldsymbol{\Sigma}(\btheta^\star) \mathbf{M}^\top \nabla^2 \bff(\btheta)) = \frac{\eta}{2} \text{Tr}\left(2\tau \mathbf{M} \nabla^2 \bff(\btheta)\right) + O(\eta^2) = \eta \tau \text{Tr}(\mathbf{M} \nabla^2 \bff(\btheta)) + O(\eta^2).
\]
Thus,
\[
\mathcal{L}^{(\eta)} \bff(\btheta) = -\nabla \bff(\btheta)^\top \mathbf{M} \nabla L(\btheta) + \tau \text{Tr}(\mathbf{M} \nabla^2 \bff(\btheta)) + O(\eta).
\]
As \(\eta \to 0\), \(\mathcal{L}^{(\eta)} \bff(\btheta)\) converges to:
\[
\mathcal{L} \bff(\btheta) = -\nabla \bff(\btheta)^\top \mathbf{M} \nabla L(\btheta) + \tau \text{Tr}(\mathbf{M} \nabla^2 \bff(\btheta)),
\]
which is the generator of the Itô SDE:
\[
d\btheta_t = -\mathbf{M} \nabla L(\btheta_t) dt + \sqrt{2\tau} \mathbf{M}^{1/2} dW_t.
\]
By the weak convergence theory (e.g., via the martingale problem or generator convergence), the process \(\btheta^{(\eta)}(t)\) converges weakly to the solution of this SDE.
\end{proof}

\begin{proposition}[Gibbs Stationary Distribution]
The SDE in \cref{eq:SDE_main} has stationary distribution:
\begin{align}
p_\infty(\btheta) &\propto \exp(-L(\btheta)/\tau) \label{eq:line5}
\end{align}
\end{proposition}

\begin{proof}
The generator of the SDE (5) is $\mathcal{L}f = -\mathbf{M}\nabla L \cdot \nabla f + \tau \text{tr}(\mathbf{M} \nabla^2 f)$. The Fokker-Planck equation for the probability density $p(t,\btheta)$ is:
\begin{align*}
\partial_t p &= \mathcal{L}^*p = \nabla \cdot (\mathbf{M} \nabla L \, p) + \tau \nabla \cdot (\mathbf{M} \nabla p)
\end{align*}
where $\mathcal{L}^*$ is the adjoint operator. Setting $\partial_t p = 0$ for stationarity:
\begin{align*}
0 &= \nabla \cdot (\mathbf{M} \nabla L \, p_\infty) + \tau \nabla \cdot (\mathbf{M} \nabla p_\infty) \\
&= \nabla \cdot (\mathbf{M} \nabla L \, p_\infty + \tau \mathbf{M} \nabla p_\infty)
\end{align*}
This implies the current $J = \mathbf{M} \nabla L \, p_\infty + \tau \mathbf{M} \nabla p_\infty$ has zero divergence. For a potential-driven system, we require $J = \mathbf{0}$:
\begin{align*}
\mathbf{M} \nabla L \, p_\infty + \tau \mathbf{M} \nabla p_\infty &= \mathbf{0} \\
\nabla L \, p_\infty + \tau \nabla p_\infty &= \mathbf{0} \quad \text{(since $\mathbf{M} \succ 0$)} \\
\frac{\nabla p_\infty}{p_\infty} &= -\frac{\nabla L}{\tau}
\end{align*}
Integrating: $\log p_\infty = -L/\tau + \text{const}$, which gives \cref{eq:line5}.
\end{proof}


\begin{theorem}[Free Energy Minimization]
Let the empirical risk $L(\btheta)$ admit multiple local minima $\{\btheta_i^\star\}_{i=1}^m$ with Hessians $\mathbf{H}(\btheta_i^\star) \succ 0$.
Under the SDE in \cref{eq:SDE_main} with temperature $\tau$, the stationary probability that training resides in basin $i$ is given by:
\begin{align}\label{free energy formula}
P_\tau(\text{basin } i) &= \frac{\exp(-F_i(\tau)/\tau)}{\sum_j \exp(-F_j(\tau)/\tau)}, \quad 
F_i(\tau) := L(\btheta_i^\star) + \frac{\tau}{2} \log\det \mathbf{H}(\btheta_i^\star). 
\end{align}
\end{theorem}

\begin{proof}
From the Gibbs distribution \cref{eq:line5}, the probability mass in basin $i$ is:
\begin{align*}
P_\tau(\text{basin } i) &= \frac{\int_{B_i} e^{-L(\btheta)/\tau} d\btheta}{\int_{\mathbb{R}^d} e^{-L(\btheta)/\tau} d\btheta}
\end{align*}

where $B_i$ is the basin of attraction around minimum $\btheta_i^\star$.

For the numerator, using the quadratic approximation $L(\btheta) = L(\btheta_i^\star) + \frac{1}{2}(\btheta - \btheta_i^\star)^\top \mathbf{H}(\btheta_i^\star) (\btheta - \btheta_i^\star)$ in basin $i$ we get:
\begin{align*}
\int_{B_i} e^{-L(\btheta)/\tau} d\btheta &= \int_{\mathbb{R}^d} \exp\left(-\frac{L(\btheta_i^\star)}{\tau} - \frac{1}{2\tau}(\btheta - \btheta_i^\star)^\top \mathbf{H}(\btheta_i^\star) (\btheta - \btheta_i^\star)\right) d\btheta \\
&= e^{-L(\btheta_i^\star)/\tau} \int_{\mathbb{R}^d} \exp\left(-\frac{1}{2\tau}(\btheta - \btheta_i^\star)^\top \mathbf{H}(\btheta_i^\star) (\btheta - \btheta_i^\star)\right) d\btheta
\end{align*}

The integral is a multivariate Gaussian with covariance $\tau\mathbf{H}(\btheta_i^\star)^{-1}$. Using the standard formula for Gaussian integrals:
\begin{align*}
\int_{\mathbb{R}^d} \exp\left(-\frac{1}{2}y^\top \mathbf{\Sigma}^{-1} y\right) dy &= (2\pi)^{d/2} (\det \mathbf{\Sigma})^{1/2}
\end{align*}
With $\mathbf{\Sigma} = \tau\mathbf{H}(\btheta_i^\star)^{-1}$, we have $\det \mathbf{\Sigma} = \tau^d (\det \mathbf{H}(\btheta_i^\star))^{-1}$ and $\mathbf{\Sigma}^{-1} = \tau^{-1}\mathbf{H}(\btheta_i^\star)$:
\begin{align*}
\int_{\mathbb{R}^d} \exp\left(-\frac{1}{2\tau}(\btheta - \btheta_i^\star)^\top \mathbf{H}(\btheta_i^\star) (\btheta - \btheta_i^\star)\right) d\btheta &= (2\pi)^{d/2} (\tau^d (\det \mathbf{H}(\btheta_i^\star))^{-1})^{1/2} \\
&= (2\pi \tau)^{d/2} (\det \mathbf{H}(\btheta_i^\star))^{-1/2}
\end{align*}

Therefore:
\begin{align*}
\int_{B_i} e^{-L(\btheta)/\tau} d\btheta &= e^{-L(\btheta_i^\star)/\tau} (2\pi \tau)^{d/2} (\det \mathbf{H}(\btheta_i^\star))^{-1/2} \\
&= (2\pi \tau)^{d/2} \exp\left(-L(\btheta_i^\star)/\tau - \frac{1}{2}\log\det \mathbf{H}(\btheta_i^\star)\right) \\
&= (2\pi \tau)^{d/2} \exp\left(-\frac{1}{\tau}\left(L(\btheta_i^\star) + \frac{\tau}{2}\log\det \mathbf{H}(\btheta_i^\star)\right)\right) \\
&= (2\pi \tau)^{d/2} \exp(-F_i(\tau)/\tau)
\end{align*}

Similarly, the total partition function is:
\begin{align*}
Z(\tau) &= \int_{\mathbb{R}^d} e^{-L(\btheta)/\tau} d\btheta =d \sum_{j=1}^m \int_{B_j} e^{-L(\btheta)/\tau} d\btheta \\
&= (2\pi \tau)^{d/2} \sum_{j=1}^m \exp(-F_j(\tau)/\tau)
\end{align*}

Therefore:
\begin{align*}
P_\tau(\text{basin } i) &= \frac{(2\pi \tau)^{d/2} \exp(-F_i(\tau)/\tau)}{(2\pi \tau)^{d/2} \sum_j \exp(-F_j(\tau)/\tau)} = \frac{\exp(-F_i(\tau)/\tau)}{\sum_j \exp(-F_j(\tau)/\tau)}
\end{align*}

This completes the proof of the free energy formula \cref{free energy formula}.
\end{proof}

\clearpage
\newpage

\section{Experimental Setups and More Results}
\label{suppl:exp}

\subsection{Experimental Setups}

\textbf{Models.}
We utilize two popular classes of LLM models for our pre-training experiments:
\begin{itemize}
[leftmargin=2.5em,topsep=1pt,itemsep=1.5pt,partopsep=1.5pt, parsep=1.5pt]
    \item {\bf GPT-2.} 
    We use GPT-2 (small)  model~\citep{radford2019language}, implemented via the nanoGPT code base ~\citep{Karpathy2022}. 
    Following nanoGPT, the model employs Gaussian Error Linear Unit (GELU) activations and standard Layer Normalization (LayerNorm). 
    Detailed model configurations are provided in \cref{tab:model config}.
    \item {\bf LLaMA.} 
    LLaMA~\citep{touvron2023llama} is another popular decoder-only Transformer architecture, incorporating Rotary Positional Encoding (RoPE)~\citep{su2024roformer}, Swish-Gated Linear Unit (SwiGLU), and Root mean square layer normalization (RMSNorm). 
    For implementation, we utilize the LLaMA code from HuggingFace Transformers Library~\citep{wolf-etal-2020-transformers}. 
    Additional model configurations are detailed in \cref{tab:model config}.
\end{itemize}

\textbf{Datasets.}
Training is performed on the \texttt{FineWeb-Edu} dataset~\citep{penedo2024the}. 
We adopt the a subset randomly sampled from the whole dataset of around 100B \texttt{GPT-2} tokens. 
The same dataset has been widely used in literature on LLM pre-training.

\textbf{Optimizers.}
To generalize our findings across different optimizers, we choose:
\begin{itemize}
[leftmargin=2.5em,topsep=1pt,itemsep=1.5pt,partopsep=1.5pt, parsep=1.5pt]
    \item \textbf{AdamW.} AdamW~\citep{kingma2014adam} is adopted with hyperparameters $\beta_1 = 0.95$, $\beta_2 = 0.95$, and weight decay $0.1$.
    \item \textbf{Muon.} Muon~\citep{jordan2024muon} is used with momentum of $0.95$ and weight decay $0.1$.
    \item \textbf{Adam-mini.} The hyperparameter of Adam-mini~\citep{zhang2024adam} is the same as AdamW.
    \item \textbf{Lion.} Lion~\citep{chen2024symbolic} is used with hyperparameters $\beta_1 = 0.95$, $\beta_2 = 0.98$. The LR of Lion $\eta$ is divided by $10\times$ compared with the LR of AdamW in the same experiments, and the weight decay $\lambda$ is ramped up to $10\times$ to keep the effective LR $\lambda \eta = 0.1$.
\end{itemize}
All these optimizers are used with gradient clipping at $1.0$ for stability.

\begin{table}[!ht]
    \centering
    \renewcommand{\arraystretch}{1.25}
    \caption{Model configurations.}
    \label{tab:model config}
    \begin{tabular}{l|c|c|c|c|c}
    \hline 
    Acronym & Size & $d_{\mathrm{model}}$ & $d_{\mathrm{FF}}$ & n$\_$head & depth \\
    \hline\hline 
    GPT-2 (small) & 124M & 768 & 3072 & 12 & 12 \\
    LLaMA (93M) & 93M & 512 & 2048 & 16 & 8\\
    LLaMA (170M) & 170M & 768 & 3072 & 12 & 8\\
    LLaMA (270M) & 270M & 1024 & 4096 & 16 & 8\\
    LLaMA (530M) & 530M & 1536 & 6144 & 24 & 8\\
    \hline 
    \end{tabular}
\end{table}

\subsection{More Results Under Various Setups.}

In this section, we extend our findings to other architectures, optimization algorithms, and larger training scales.
Due to computational constraints, we primarily focus on validating the sharp-to-flat early dynamics and the proposed BS scheduling principle across these settings.
We also explore the warmup–tuning recipe on additional architectures.

\begin{figure}[tb!]
    \centering
    \includegraphics[width=0.81\textwidth]{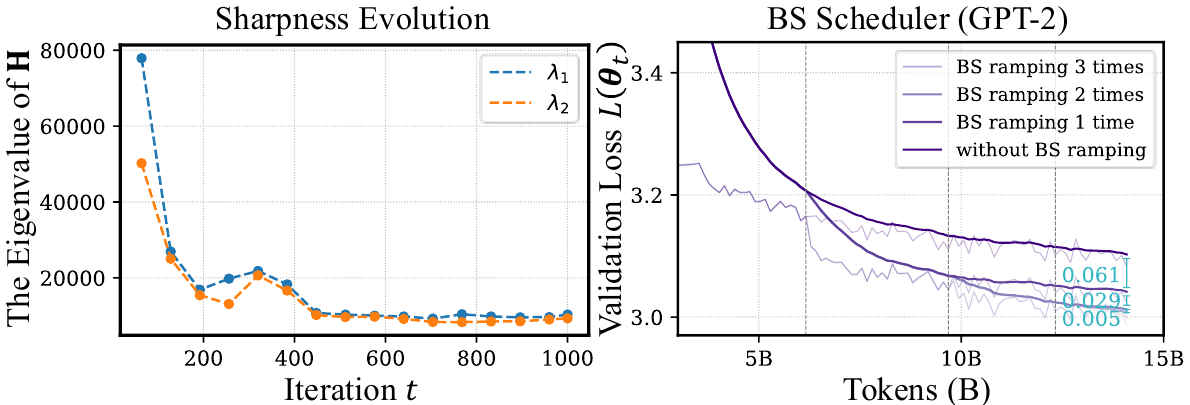}
    \caption{
    \textbf{Extensive to \texttt{GPT-2} architectures.}
    \textbf{(Left). Early sharp-to-flat dynamics.}
    Evolution of the top eigenvalues of the Hessian across iterations: $\lambda_i(\mathbf{H}(\btheta_t))$ vs. iteration $t$.
    \textbf{(Right). BS scheduling improves data efficiency.}
    \textbf{BS scheduling improves data efficiency.}
    We use a BS schedule that starts at $0.49$M and increases by $4\times$ at each ramp. 
    Models are trained with $1, 2$, or $3$ ramping steps, while models without ramping serve as the baseline. 
    Vertical gray dashed lines indicate ramping positions.
    }
    \label{fig:gpt_parta}
    \vspace{-5pt}
\end{figure}

\begin{figure}[tb!]
    \centering
    \includegraphics[width=0.47\textwidth]{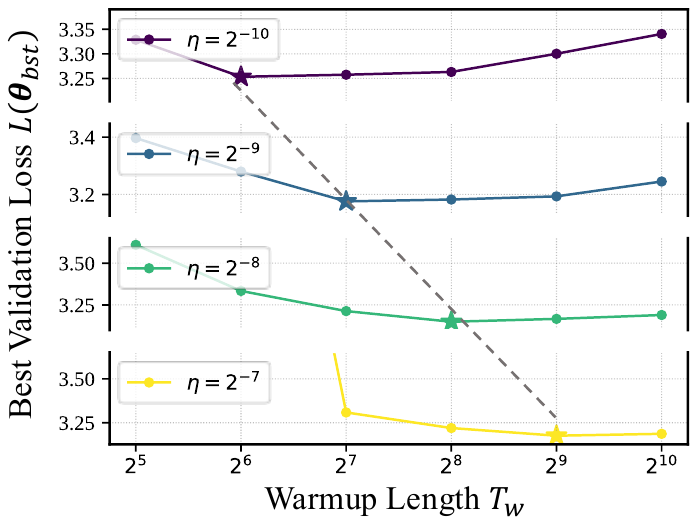}
    \caption{
    \textbf{Extensive to \texttt{GPT-2} architectures.}
    \textbf{Larger Peak LR, Longer Warmup.}
    We train a series of \texttt{GPT-2} models with 100 TPP.
    We vary the peak LRs $\eta$ and warmup lengths $T_w$. 
    We plot the best validation loss $L(\btheta_{\text{bst}})$ vs. $T_w$ for different $\eta$.
    The optimal $T_w$ is highlighted with a star.
    }
    \label{fig:gpt_partb}
    \vspace{-5pt}
\end{figure}

\textbf{Extension to \texttt{GPT-2} Architectures.}
See \cref{fig:gpt_parta,fig:gpt_partb} for details.

\begin{figure}[tb!]
    \centering
    \includegraphics[width=0.81\textwidth]{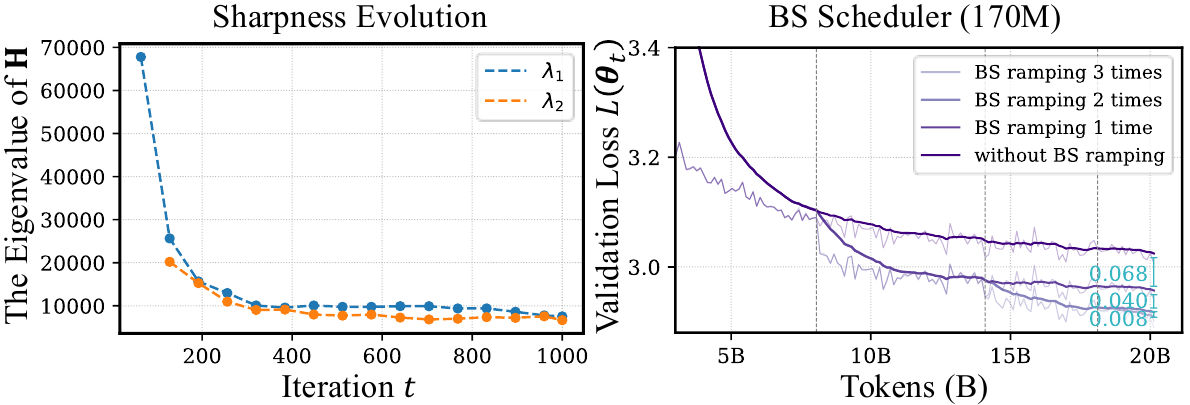}
    \caption{
    \textbf{Extensive to Adam-mini optimizer.}
    \textbf{(Left). Early sharp-to-flat dynamics.}
    Evolution of the top eigenvalues of the Hessian across iterations: $\lambda_i(\mathbf{H}(\btheta_t))$ vs. iteration $t$.
    \textbf{(Right). BS scheduling improves data efficiency.}
    \textbf{BS scheduling improves data efficiency.}
    We use a BS schedule that starts at $0.49$M and increases by $4\times$ at each ramp. 
    Models are trained with $1, 2$, or $3$ ramping steps, while models without ramping serve as the baseline. 
    Vertical gray dashed lines indicate ramping positions.
    }
    \label{fig:adam_mini}
    \vspace{-5pt}
\end{figure}

\begin{figure}[tb!]
    \centering
    \includegraphics[width=0.81\textwidth]{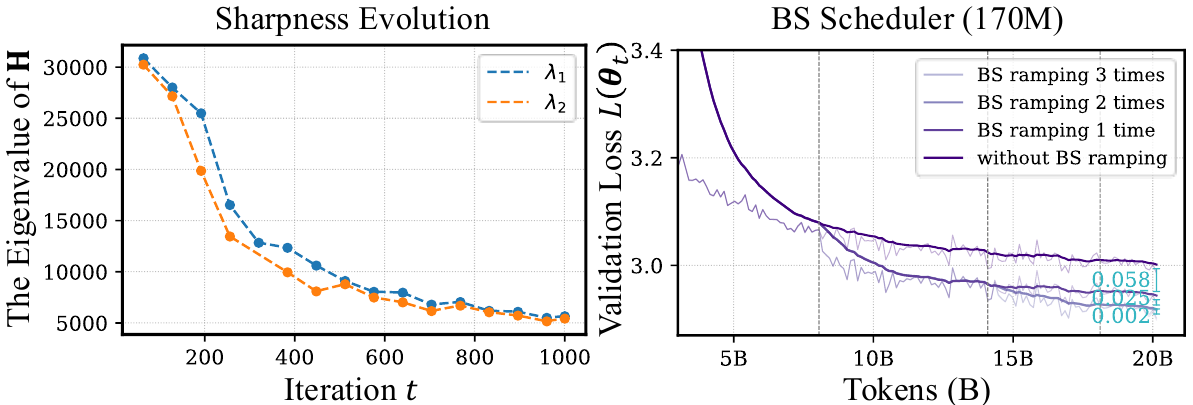}
    \caption{
    \textbf{Extensive to Lion optimizer.}
    \textbf{(Left). Early sharp-to-flat dynamics.}
    Evolution of the top eigenvalues of the Hessian across iterations: $\lambda_i(\mathbf{H}(\btheta_t))$ vs. iteration $t$.
    \textbf{(Right). BS scheduling improves data efficiency.}
    \textbf{BS scheduling improves data efficiency.}
    We use a BS schedule that starts at $0.49$M and increases by $4\times$ at each ramp. 
    Models are trained with $1, 2$, or $3$ ramping steps, while models without ramping serve as the baseline. 
    Vertical gray dashed lines indicate ramping positions.
    }
    \label{fig:lion}
    \vspace{-5pt}
\end{figure}

\textbf{Extension to Other Optimizers.}
See \cref{fig:adam_mini} for Adam-mini, and see \cref{fig:lion} for Lion.

\begin{figure}[tb!]
    \centering
    \includegraphics[width=0.81\textwidth]{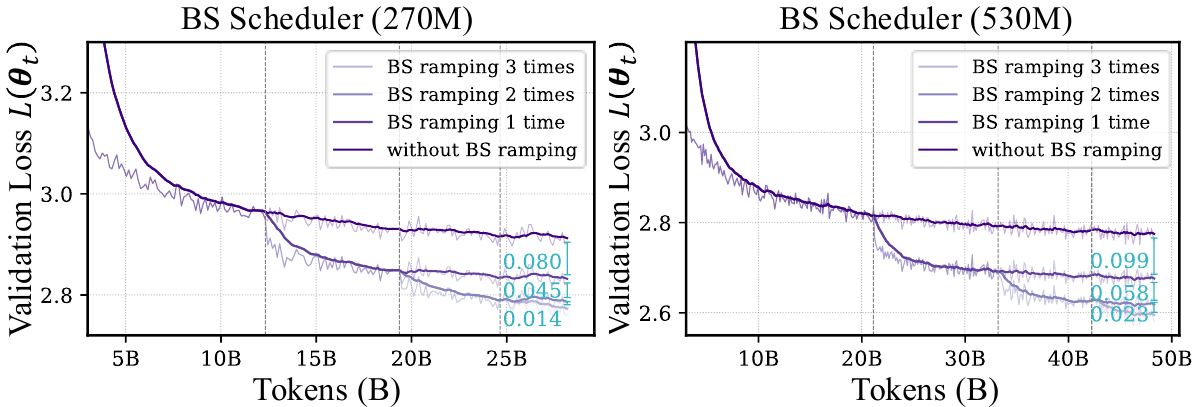}
    \caption{
    \textbf{Extensive to Larger Scale.}
    \textbf{(Left). Early sharp-to-flat dynamics.}
    Evolution of the top eigenvalues of the Hessian across iterations: $\lambda_i(\mathbf{H}(\btheta_t))$ vs. iteration $t$.
    \textbf{(Right). BS scheduling improves data efficiency.}
    \textbf{BS scheduling improves data efficiency.}
    We use a BS schedule that starts at $0.49$M and increases by $4\times$ at each ramp. 
    Models are trained with $1, 2$, or $3$ ramping steps, while models without ramping serve as the baseline. 
    Vertical gray dashed lines indicate ramping positions.
    }
    \label{fig:larger}
    \vspace{-5pt}
\end{figure}

\textbf{Extension to Larger Models.}
See \cref{fig:larger} for models with 270M and 530M parameters.

\subsection{Ablation Studies on Early Instabilities.}

\begin{figure}[tb!]
    \centering
    \includegraphics[width=0.81\textwidth]{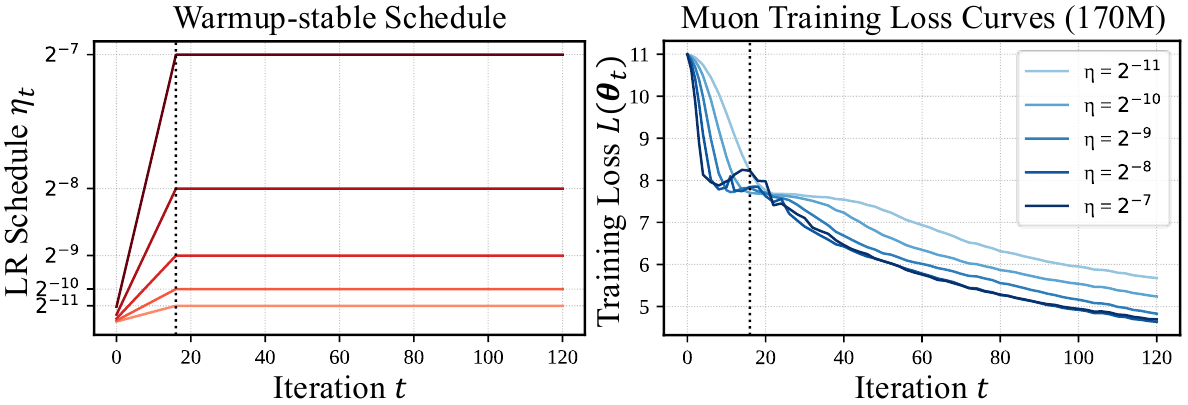}
    \caption{
    \textbf{Muon consistently shows loss spikes and plateaus early in training.}
    We train a series of \texttt{LLaMA-2} models with $170$M parameters.
    We adopt a warmup-stable schedule, where the warmup length is shortened to $16$ iterations and the peak LR is varied, $\eta \in \{2^{-11}, 2^{-10}, 2^{-9}, 2^{-8}, 2^{-7}\}$. 
    \textbf{(Left).} LR schedule: $\eta_t$ vs. training iteration $t$. 
    \textbf{(Middle, Right).} Training loss curves for different model sizes: $L(\btheta_t)$ vs. training iteration $t$. 
    The vertical dashed line marks the end of the warmup phase.
    }
    \label{fig:muon}
    \vspace{-5pt}
\end{figure}

\begin{figure}[tb!]
    \centering
    \includegraphics[width=0.81\textwidth]{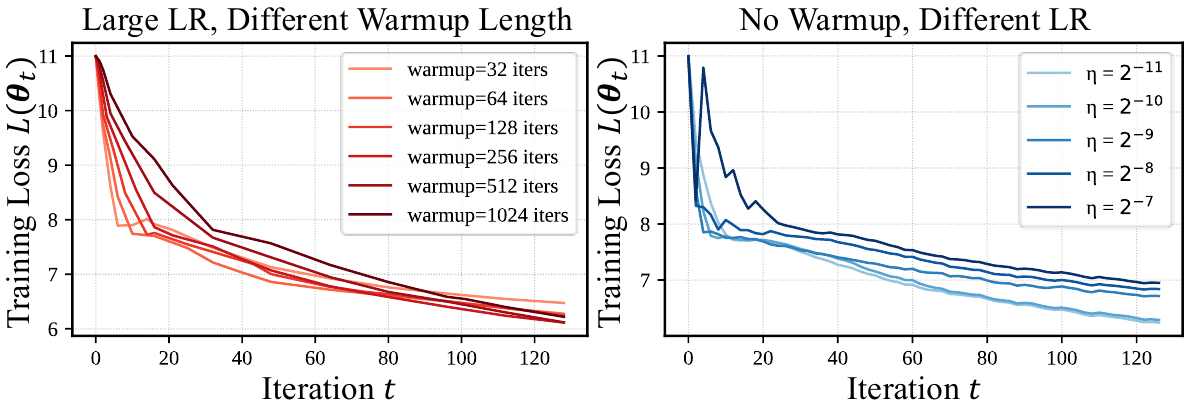}
    \caption{
    \textbf{(Left.) Shorter warmup, more loss spikes.}
    We train a series of \texttt{LLaMA-2} models with $170$M parameters.
    We adopt a warmup-stable schedule, where the warmup length varies from $\{2^{5}, 2^{6}, 2^{7}, 2^{8}, 2^{9}, 2^{10}\}$ iterations  and the peak LR is fixed $\eta = 2^{-7}$. 
    \textbf{(Right). Zero warmup and small BS  leads to larger loss spikes.} 
    We train a series of \texttt{LLaMA-2} models with $170$M parameters.
    We adopt a constant LR schedule, where no warmup and the peak LR is varied, $\eta \in \{2^{-11}, 2^{-10}, 2^{-9}, 2^{-8}, 2^{-7}\}$. 
    We also use the $0.49$M BS.
    }
    \label{fig:warmup}
    \vspace{-5pt}
\end{figure}

In this section, we conduct ablation studies on the root cause of instabilities, such as loss spikes and plateaus, observed in early training.

\textbf{Is it the unstable optimizer?}
To disentangle optimizer-induced instability from landscape-induced instability, we repeated the experiments using Muon, a substantially more stable optimizer than AdamW.
In \cref{fig:muon}, the loss spikes and plateaus consistently occurs under Muon when warmup is shortened or the peak LR is increased. This rules out the possibility that the behavior stems from AdamW’s startup issues.

\textbf{What if we use longer warmup?}
We vary only the warmup length while fixing the peak LR at $2^{-7}$.
In \cref{fig:warmup}~(left), shorter warmup lengths lead to higher possibilities of loss spikes
This behavior is consistent with the sharp-to-flat dynamics. Early in training, the model resides in sharper regions of the landscape, where only sufficiently small LRs ensure stable updates. Therefore, warmup is needed to gradually increase the LR until the trajectory enters flatter regions that can tolerate larger LRs.

\textbf{What if zero warmup and small BS?}
We also conduct experiments with no warmup and small batch size. 
In \cref{fig:warmup}~(right), the loss spikes become even more significant. This aligns with our explanation: with no warmup, the LR jumps immediately to a large value while sharpness is still extremely high, causing a spike almost at initialization. 
More importantly, small BS does not replace warmup, and the instability still appears because of the sharpness.

\end{document}